\documentclass{article}

\usepackage{microtype}
\usepackage{graphicx}
\usepackage{subfigure}
\usepackage{booktabs} 

\usepackage{hyperref}

\usepackage[accepted]{icml2023}

\usepackage{amsmath}
\usepackage{amssymb}
\usepackage{mathtools}
\usepackage{amsthm}

\usepackage[capitalize,noabbrev]{cleveref}

\theoremstyle{plain}
\newtheorem{theorem}{Theorem}[section]
\newtheorem{proposition}[theorem]{Proposition}
\newtheorem{lemma}[theorem]{Lemma}
\newtheorem{corollary}[theorem]{Corollary}
\newtheorem{observation}[theorem]{Observation}
\theoremstyle{definition}
\newtheorem{definition}[theorem]{Definition}

\theoremstyle{remark}

\usepackage[textsize=tiny]{todonotes}
\usepackage{caption}

\usepackage{amsmath}\allowdisplaybreaks
\usepackage{amsfonts,bm}
\usepackage{amssymb}
\usepackage{amsthm}
\usepackage{amsmath}
\usepackage{algorithm}
\usepackage{algorithmic}
\usepackage{multirow}
\usepackage{booktabs}
\usepackage{xcolor}

\newcommand{\newterm}[1]{{\bf #1}}

\def\eqref#1{equation~(\ref{#1})}

\def\1{\bf{1}}
\newcommand{\train}{\mathcal{D}}

\def\BN{{\mathbb{N}}}

\def\BR{{\mathbb{R}}}

\newcommand{\rank}{\mathrm{rank}}

\DeclareMathOperator{\Tr}{Tr}

\makeatletter
\def\Ddots{\mathinner{\mkern1mu\raise\p@
\vbox{\kern7\p@\hbox{.}}\mkern2mu
\raise4\p@\hbox{.}\mkern2mu\raise7\p@\hbox{.}\mkern1mu}}
\makeatother

\makeatletter
\newcommand*{\rom}[1]{\expandafter\@slowromancap\romannumeral #1@}
\makeatother

\icmltitlerunning{Quantifying the Variability Collapse of Neural Networks}

\begin{document}

\twocolumn[
\icmltitle{Quantifying the Variability Collapse of Neural Networks}



\icmlsetsymbol{equal}{*}

\begin{icmlauthorlist}
\icmlauthor{Jing Xu}{equal,yyy}
\icmlauthor{Haoxiong Liu}{equal,yyy}
\end{icmlauthorlist}

\icmlaffiliation{yyy}{Institute for Interdisciplinary Information Sciences, Tsinghua University, Beijing, China}

\icmlcorrespondingauthor{Jing Xu}{xujing21@mails.tsinghua.edu.cn}
\icmlcorrespondingauthor{Haoxiong Liu}{liuhx20@mails.tsinghua.edu.cn}

\icmlkeywords{Machine Learning, ICML}

\vskip 0.3in
]



\printAffiliationsAndNotice{\icmlEqualContribution} 

\begin{abstract}
Recent studies empirically demonstrate the positive relationship between the transferability of neural networks and the within-class variation of the last layer features.
The recently discovered Neural Collapse~(NC) phenomenon provides a new perspective of understanding such last layer geometry of neural networks. In this paper, we propose a novel metric, named Variability Collapse Index~(VCI), to quantify the variability collapse phenomenon in the NC paradigm. The VCI metric is well-motivated and intrinsically related to the linear probing loss on the last layer features. Moreover, it enjoys desired theoretical and empirical properties, including invariance under invertible linear transformations and numerical stability, that distinguishes it from previous metrics. Our experiments verify that VCI is indicative of the variability collapse and the transferability of pretrained neural networks. 
\end{abstract}

\section{Introduction}
The pursuit of powerful models capable of extracting features from raw data and performing well on downstream tasks has been a constant endeavor in the machine learning community~\citep{bommasani2021opportunities}. In the past few years, researchers have developed various pretraining methods~\citep{chen2020simple,khosla2020supervised,grill2020bootstrap,he2022masked,baevski2022data2vec} that  enable models to learn from massive real world datasets. However, there is still a lack of systematic understanding regarding the transferability of deep neural networks, \emph{i.e.}, whether they can leverage the information in the pretraining datasets to achieve high performance in downstream tasks~\citep{abnar2021exploring,fang2023does}.

The performance of a pretrained model is closely related to the quality of the features it produces. The recently proposed concept of \textit{neural collapse~(NC)}~\citep{prevalence} provides a paradigmatic way to study the representation of neural networks. 
According to neural collapse, the last layer features of neural networks adhere to the following rule of 
\textit{variability collapse~(NC1)}: As the training proceeds, the representation of a data point converges to its corresponding class mean. Consequently,  the within-class variation of the features converges to zero.


The deep connection between transferability and neural collapse is rooted in the variability collapse criterion. Previous works~\citep{feng2021rethinking,kornblith2021better,sariyildiz2022improving} empirically find that although models with collapsed last-layer feature representations exhibit better pretraining accuracy, they tend to yield worse performance for downstream tasks. These works give an intuitive explanation that pushing the feature to their class means results in the loss of the diverse structures useful for downstream tasks.
Building upon this understanding, researchers design various algorithms~\citep{jing2021understanding,kini2021label,chen2022perfectly,dubois2022improving, sariyildiz2023no} that either explicitly or implicitly levarage the variability collapse criterion to retain the feature diversity in the pretraining phase, and thereby improve the transferability of the models. 

Straightforward as it is stated, the variability criterion is still not thoroughly understood. One fundamental question is how to mathematically quantify variability collapse. Previous works propose variability collapse metrics that are meaningful in specific settings~\citep{prevalence,zhu2021geometric,kornblith2021better,hui2022limitations}. However, a more principled characterization is required when we want to use variability collapse to analyze transferability. For example, in the linear probing setting, the loss function is invariant to invertible linear transformations on the last layer features. Consequently, it is reasonable to expect that the collapse metric of the features would also be invariant under such transformations, in order to properly reflect the models performance on downstream tasks. However, as we will point out in Section~\ref{sec: previous invariance}, no previous metric can achieve this high level of invariance, to the best of the authors' knowledge.

To obtain a well-motivated and well-defined variability collapse metric, we tackle the problem from a loss minimization perspective.
Our analysis reveals that the minimum mean squared error (MSE) loss in linear probing on a set of pretrained features can be expressed concisely, with a major component being $\Tr[\Sigma_T^{\dagger}\Sigma_B]$. Here, $\Sigma_B$ is the between-class feature covariance matrix and $\Sigma_T$ is the overall feature covariance matrix, as defined in Section~\ref{sec: setup}.
This term serves as an indicator of variability collapse, since it achieves its maximum $\rank(\Sigma_B)$ for fully collapse configurations where the feature of each data point coincides with the feature class mean. Furthermore, an important implication of its connection with MSE loss is that the invertible linear transformation invariance of the loss function directly transfers to the quantity $\Tr[\Sigma_T^{\dagger}\Sigma_B]$.

Motivated by the above investigations, we propose the following collapse metric, which we name \textbf{Variability Collapse Index~(VCI)}:
\begin{align*}
    \text{VCI}=1-\frac{\Tr[\Sigma_{T}^\dagger \Sigma_{B}]}{\rank(\Sigma_B)}.
\end{align*}
The VCI metric possesses the desirable property of invariance under invertible linear transformation, making it a proper indicator of last layer representation collapse.
Furthermore, VCI enjoys a higher level of numerical stability compared previous collapse metrics. 
We conduct extensive experiments to validate the effectiveness of the proposed VCI metric. The results show that VCI is a valid index for variability collapse across different architectures. We also show that VCI has a strong correlation with accuracy of various downstream tasks, and serves as a better index for transferability compared with existing metrics.




\section{Related Works}
\paragraph{Neural Collapse.}
The seminal paper~\citet{prevalence} proposes the concept of neural collapse, which consists of four paradigmatic criteria that govern the terminal phase of training of neural networks. 

One research direction regarding neural collapse focuses on rigorously proving neural collapse for specific learning models. A large portion of them adopt the layer peeled model~\citep{mixon2020neural,fang2021exploring}, which treats the last layer feature vector as unconstrained optimization variables. In this setting, both cross entropy loss~\citep{lu2020neural,zhu2021geometric,ji2021unconstrained} and mean square loss~\citep{tirer2022extended,zhou2022optimization} exhibit neural collapse configurations as the only global minimizers and have benign optimization landscapes. Additionally, other theoretical investigations explore neural collapse from the perspective of optimization dynamics~\citep{han2021neural}, max margin~\citep{zhou2022learning} and more generalized setting~\citep{nguyen2022memorization,tirer2022perturbation,zhou2022all,yaras2022neural}

Another research direction draws inspiration from  the neural collapse phenomenon to devise training algorithms. For instance, some studies empirically demonstrate that fixing the last-layer weights of neural networks to an Equiangular Tight Frame~(ETF) reduces memory usage~\citep{zhu2021geometric},  and improves the performance on imbalanced dataset~\citep{yang2022we,thrampoulidis2022imbalance, zhu2022balanced} and few shot learning tasks~\citep{yang2023neural}.

\paragraph{Representation Collapse and Transferability.}
Understanding and improving the transferability of neural networks to unknown tasks have attracted significant attention in recent years~\citep{tan2018survey,ruder2019transfer,zhuang2020comprehensive}.
Previous works~\citep{feng2021rethinking, sariyildiz2022improving, cui2022discriminability} empirically demonstrate that the diversity of last layer features is positively correlated with the transferability of neural networks, highlighting a tradeoff between pretraining accuracy and transfer accuracy.
To address this challenge, various methods~\citep{schilling2021quantifying,touvron2021grafit, xie2022hidden} have been proposed to quantify and mitigate representation collapse. For example,~\citet{kornblith2021better} show that using a low temperature for softmax activation in training reduces class separation and improves  transferability. 
Neural collapse provides a novel perspective for understanding this fundamental tradeoff~\citep{galanti2021role,lideep}. Notably,~\citet{hui2022limitations} reveal that neural collapse can be at odds with transferability by causing a loss of crucial information necessary for downstream tasks.

\section{Preliminaries}\label{sec: preliminiaries}
\subsection{Notations and Problem Setup}\label{sec: setup}
Throughout this paper, we adopt the following notation conventions. We use $\|v\|$ to denote Euclidean norm of vector $ v\in \BR^d$. We use $\|A \|_F$ to denote Frobenious norm and $A^\dagger$ to denote the pseudo-inverse of  matrix $A\in\BR^{d\times d}$, $d\in\BN_{+}$. We use $[n]$ as a short hand for $\{1,\cdots n\}$. We use $e_k\in\BR^K$ to denote the vector whose $k$-th entry is $1$ and the other entries are $0$.
We use $\mathbf{1}_d$ and $\mathbf{0}_d$ to denote the all-one and the zero vector in $\BR^{d}$, and use $\mathbf{I}_{d\times d}$ and $\mathbf{0}_{d\times d}$ to denote the identity matrix and the zero matrix in $\BR^{d\times d}$. We omit the subscripts of dimension when the context is clear. 

Consider a $K$-class classification problem
on a balanced dataset $\train=\{(x_{k,i}, e_k)\}_{ k \in[K],i\in[N]}$, where $N$ is the number of samples from each class. It is worth noting that the results presented in this paper can be readily extended to imbalanced datasets. Each sample consists of a data point $x_{k,i} \in \BR^d$ and an one-hot label $e_k\in \BR^K$. The classifier $W \phi(\cdot)+b$ is composed  of a feature extractor $\phi: \mathbb{R}^d\to \BR^p$ and a linear layer with $W\in \mathbb{R}^{K\times p}$ and $b\in \BR^K$. Let $h_{k, i}=\phi(x_{k,i})$ denote the feature vector of $x_{k, i}$, and $H=(h_{k,i})_{k\in[K],i\in[N]} \in \mathbb{R}^{p\times KN}$ denote the feature matrix. 
The feature extractor can be any pretrained neural network, till its penultimate layer. 

For a given feature matrix, we denote $\mu_k(H)=(1/N)\sum_{i\in[N]}h_{k,i}$ as the $k$-th class mean, 
and $\mu_G(H)=(1/KN)\sum_{k\in[K],i\in[N]} h_{k,i}$ as the global mean.
Throughout this paper, we will frequently refer to the following notions of feature covariance. Specifically, we denote the within-class covariance matrix by
\begin{equation}\label{eq: within}
\Sigma_W (H) = \frac{1}{KN}\sum_{k\in[K]}\sum_{i\in[N]} (h_{k, i} - \mu_k)(h_{k, i} - \mu_k)^\top,
\end{equation}
and the between-class covariance matrix by
\begin{equation}\label{eq: between}
\Sigma_B (H) = \frac{1}{K}\sum_{k\in[K]} (\mu_k - \mu_G)(\mu_k - \mu_G)^\top.
\end{equation}
The overall covariance matrix is defined as
\begin{equation}\label{eq: total}
\Sigma_T (H) = \frac{1}{KN}\sum_{k\in[K]}\sum_{i\in[N]} (h_{k, i} - \mu_G)(h_{k, i} - \mu_G)^\top.
\end{equation}
A bias-variance decomposition argument gives $\Sigma_T(H) = \Sigma_B(H) + \Sigma_W(H)$, whose proof is provided in Equation~\ref{eq: decomposition} for completeness. We omit the feature matrix $H$ in the above notations, when the context is clear.

We define $V_B=\text{span}\{\mu_1-\mu_G,\cdots \mu_k-\mu_G\}$ as the column space of $\Sigma_B$. In the same way, we can define $V_W, V_T$ as the column space of $\Sigma_W$ and $\Sigma_T$, respectively.

\subsection{Previous Collapse Metrics}\label{sec: previous metrics}
The first item in the Neural collapse paradigm is referred to as the variability collapse criterion~(NC1), which states that as the training proceeds, the within-class variation of the last layer features will diminish and the features will concentrate to the corresponding class means.
Use the quantities defined above, NC1 happens if $\Sigma_W\to \mathbf{0}$.
In the related literature, researchers propose various ways to non-asymptotically characterize NC1.




\paragraph{Fuzziness.} One of the commonly adopted metrics for NC1 is  the normalized within-class covariance $\Tr[\Sigma_B^\dagger \Sigma_W]$~\citep{prevalence, zhu2021geometric,tirer2022extended}. The term is commonly referred to as \textit{Separation Fuzziness} or simply \textit{Fuzziness} in the related literature~\citep{he2022law}, and is inherently related to the fisher discriminant ratio~\citep{zarka2020separation}. 

\paragraph{Squared Distance.} \citet{hui2022limitations} uses the quantity 
\begin{equation}\label{eq: mean var}
\frac{\sum_{k\in[K]}\sum_{i\in[N]}\|h_{k,i}-\mu_k\|^2}{N\sum_{k\in[K]}\|\mu_k-\mu_G\|^2}
\end{equation}
to characterize NC1. 
In this paper, we refer to it as \textit{Squared Distance} for convenience. Unlike fuzziness, square distance disregards the structure of the covariance matrix and uses the ratio of the square norm between the within-class variation and the between-class variation as a measure of collapse metric.

\paragraph{Cosine Similarity.} 
\citet{kornblith2021better} uses the ratio of the average within-class cosine similarity to the overall cosine similarity to measure the dispersion of feature vectors. Define $\operatorname{sim}(x,y)=x^\top y /\left(\|x\|\|y\|\right)$ as the cosine similarity between vectors. Denote the within-class cosine distance and overall cosine distance as
\begin{align*}
    &\bar{d}_{\text {within }}=\sum_{k=1}^K \sum_{i=1}^{N} \sum_{j=1}^{N} \frac{1-\operatorname{sim}\left({h}_{k, i}, {h}_{k, j}\right)}{K N^2}, \\
    &\bar{d}_{\text {total }}=\sum_{k=1}^K \sum_{l=1}^K \sum_{i=1}^{N} \sum_{j=1}^{N} \frac{1-\operatorname{sim}\left({h}_{k, i}, {h}_{l, j}\right)}{K^2 N^2}.
\end{align*}
They refer to the term $1-\bar{d}_{\text {within }} / \bar{d}_{\text {total }}$ as \textit{class separation}.
They also propose a simplified quantity $1 - \bar{d}_{\text {within }} $, and empirically show that both of them have a negative correlation with linear probing transfer performance across different settings.
In this paper, we adopt $\bar{d}_{\text {within }}$ as the baseline metric in~\citet{kornblith2021better}, and call it \textit{Cosine Similarity} for brevity.


\section{What is an  Appropriate Variability Collapse Metric?}\label{sec: appropriate}
In this section, we explore the essential properties that a valid variability collapse metric should and should not have.

\subsection{Do Last Layer Features Fully Collapse?}\label{sec: not collapse}
The original NC1 argument states that the within class covariance converges to zero, \emph{i.e.}, $\Sigma_W\to 0$, as the training proceeds. This implies that a collapse metric should achieve minimum or maximum at these \textit{fully collapsed} configurations with $\Sigma_W=0$. 

However, the following proposition shows that the opposite is not true, \emph{i.e.},  full collapse is not necessary for loss minimization.

\begin{proposition}\label{prop: not collapse}
Consider a loss function $\ell:\BR^k\times \BR^k\to \BR$. Define the training loss as 
\begin{align}\label{eq: loss}
    L(W,b,H)&=\frac{1}{KN}\sum_{k\in[K]}\sum_{i\in[N]} \ell(Wh_{k,i}+b,e_k) \nonumber \\ 
    &\quad+\frac{\lambda_W}{2}\|W\|_F^2+\frac{\lambda_b}{2}\|b\|^2,
\end{align}
where $\lambda_W, \lambda_b\ge 0$ are regularization parameters.
Suppose that $p> K$, $N\ge 2$.
Then for any constant $C>0$, 
there exists an $H^\prime$, such that $L(W,b,H^\prime)=L(W,b,H)$, $\Sigma_B(H^\prime)=\Sigma_B(H)$, but $\|\Sigma_W(H^\prime)\|_F>C$.

\end{proposition}





{

\begin{figure}[t]
\begin{center}
\centerline{\includegraphics[width=0.7\columnwidth]{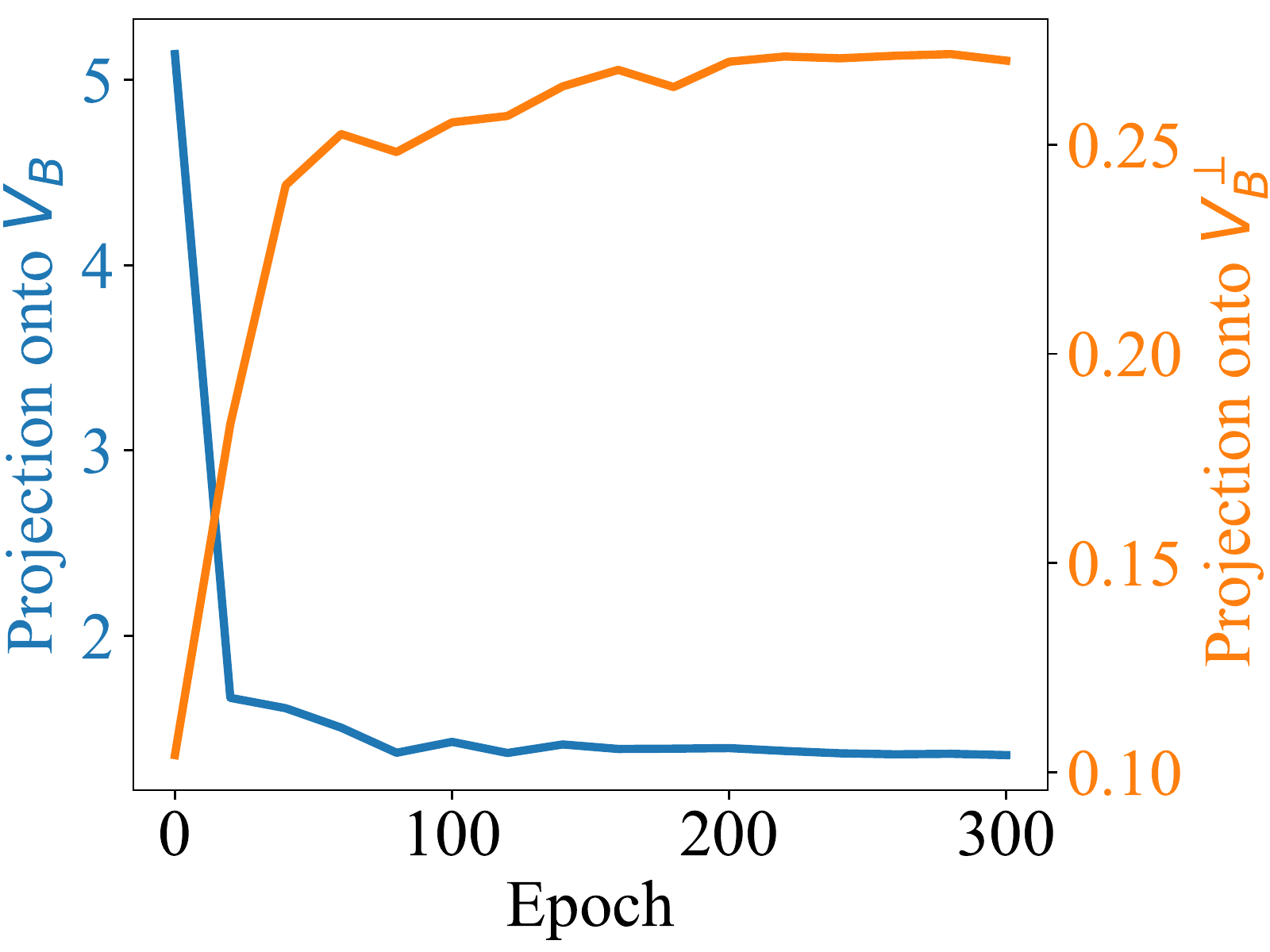}}
\captionsetup{belowskip=-15pt}
\caption{\textbf{Projections of Squared Distance onto $V_B$ and $V_B^\perp$ show opposite trends as the training proceeds.} The model is a ResNet50 trained on ImageNet-1K, using the setting specified in~\cref{sec: exp_setup}.}
\label{fig: rn50_basic_sqinout}
\end{center}
\end{figure}
}
The proof of the proposition is provided in Appendix~\ref{apx: not collapse}.
It is worth noting that  the above proposition does not contradict previous conclusions that ETF configurations are the only minimizers~\citep{zhu2021geometric,tirer2022extended}, since they require feature regularization $(\lambda_H/2)\|H\|_F^2$ in the loss function.

Our experiments show that Proposition~\ref{prop: not collapse} truly reflects the trend of neural network training. We train a {ResNet50 model on ImageNet-1K} dataset, and decompose $\Sigma_W$ into the  $V_B$ part and the $V_B^\perp$ part by computing $(1/KN)\sum_{k\in[K],i\in[N]}\|\text{Proj}_{V_B}(h_{k,i}-\mu_k)\|^2$ and $(1/KN)\sum_{k\in[K],i\in[N]}\|\text{Proj}_{V_B^\perp}(h_{k,i}-\mu_k)\|^2$. The results are shown in Figure~\ref{fig: rn50_basic_sqinout}. We observe that although the $V_B$ part steadily decreases, the $V_B^\perp$ part keeps increasing in the training process. Therefore, $\Sigma_W\to \mathbf{0}$ may not occur for real world neural network training.

Proposition~\ref{prop: not collapse} and Figure~\ref{fig: rn50_basic_sqinout} show that the last layer of neural networks exhibits high flexibility due to overparameterization. Consequently, it is unrealistic to expect standard empirical risk minimization training to achieve fully collapsed last layer representation, unless additional inductive bias are introduced. Therefore, requiring that the collapse metric reaches its minima \textit{only} at fully collapsed configurations, such as Squared Distance, will be too stringent for practical use.

\subsection{Invariance to Invertible Linear Transformations Matters}\label{sec: previous invariance}
Symmetry and invariance is a core concept in deep learning~\citep{gens2014deep,tan2018survey,chen2019invariance}.
The collapse metric discussed in Section~\ref{sec: previous metrics} enjoy certain level of invariance properties.

\begin{observation}
The Fuzziness metric $\Tr[\Sigma_B^\dagger \Sigma_W]$ is invariant to invertible linear transformation $U\in\BR^{p\times p}$ that can be decomposed into two separate transformations in $V_B$ and $V_B^\perp$.
The claim comes from the fact that
    \begin{align*}
    &\quad \Tr\left[\left(U\Sigma_B U^\top\right)^\dagger U\Sigma_WU^\top\right]\\&=\Tr\left[U^{-1,\top}\Sigma_B^\dagger U^{-1}U\Sigma_WU^\top\right]\\&=\Tr\left[\Sigma_B^\dagger \Sigma_W\right].
    \end{align*}

However, Fuzziness is not invariant to all invertible linear transformations in $\BR^p$. A simple counter example is $\Sigma_B=\left[\begin{array}{ll}
1 & 0 \\
0 & 0
\end{array}\right]$, $\Sigma_W=\left[\begin{array}{ll}
1 & 0 \\
0 & 1
\end{array}\right]$, and the linear transformation $U=\left[\begin{array}{ll}
1 & 1 \\
0 & 1
\end{array}\right]$. It can be calculated that $\Tr\left[\left(U\Sigma_B U^\top\right)^\dagger U\Sigma_W U^\top\right]=2\neq 1= \Tr\left[\Sigma_B^\dagger \Sigma_W \right]$.
\end{observation}

\begin{observation}
    The Squared Distance metric in Equation~\ref{eq: mean var} is invariant to isotropic scaling and orthogonal transformation on the feature vectors, \emph{i.e.}, since such transformations preserve the pairwise distance between the feature vectors. However, it is not invariant to invertible linear transformations in $\BR^p$.
\end{observation}

\begin{observation}
     The Cosine Similarity metric is invariant to independent scaling of each $h_{k,i}$. It is also invariant to orthogonal transformation in $\BR^p$, as such transformations preserves the cosine similarity between feature vectors. But it is easy to see that Cosine Similarity is not invariant to invertible linear transformation in $\BR^p$.
\end{observation}

However, the next proposition shows that the linear probing loss of the last layer feature is invariant under a much more general class of transformations.

\begin{observation}
    \label{prop: loss invariance}
    The minimum value of loss function in Equation~\ref{eq: loss} is invariant to invertible linear transformations on the feature vector, i.e. 
    \begin{align*}
        \min_{W,b}L(W,b,H)=\min_{W,b}L(W,b,VH),
    \end{align*}
    for any invertible $V\in \BR^{p\times p}$.
\end{observation}

In other words, if we have two pretrained models $\phi_1(\cdot)$ and $\phi_2(\cdot)$, and there exists an invertible linear transformation  $V\in \BR^{p\times p}$ such that $\phi_1(x)=V\phi_2(x)$ for any $x\in \BR^d$, then $\phi_1(\cdot)$ and $\phi_2(\cdot)$ will have exactly the same linear probing loss on any downstream data distribution. 
Therefore, when considering a collapse metric that may serve as an indicator of transfer accuracy, it is desirable for the metric to exhibit invariance to invertible linear transformations. However, as discussed previously, the metrics listed in Section~\ref{sec: previous metrics} do not possess this level of invariance. 


\subsection{Numerical Stability Issues}\label{sec: numerical stability}

\begin{figure}[t]
\begin{center}
\centerline{\includegraphics[width=0.7\columnwidth]{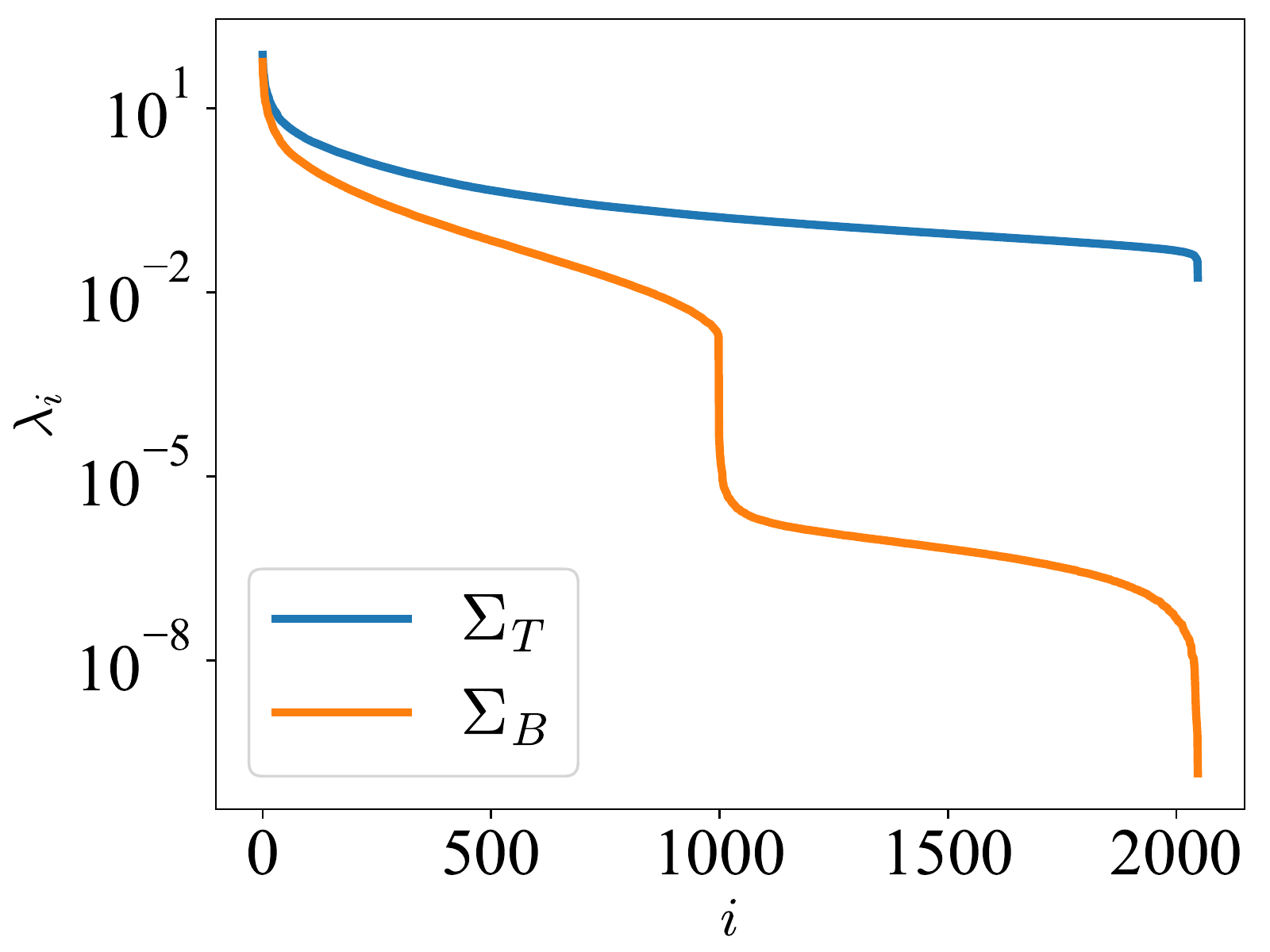}}
\captionsetup{skip=5pt}
\captionsetup{belowskip=-20pt}
\caption{\textbf{The Eigenvalue spectra of $\Sigma_B$ and $\Sigma_T$.} The spectrum of $\Sigma_T$ has a substantially larger scale. The model is a ResNet50 trained on ImageNet-1K, using the setting specified in~\cref{sec: exp_setup}.}
\label{fig: svals_rn50_basic}
\end{center}
\end{figure}

Numerical stability is an essential property for the collapse metric to ensure its practical usability.
Unfortunately, the Fuzziness metric is prone to numerical instability, primarily due to the pseudoinverse operation applied to $\Sigma_B$.

Firstly, the between-class covariance matrix $\Sigma_B$ is singular when $K\le p$, and its rank is unknown.
Due to computational imprecision, its zero eigenvalues are always occupied with small nonzero values. In the default PyTorch~\citep{paszke2019pytorch} implementation, the pseudoinverse operation includes a thresholding step to eliminate the spurious nonzero eigenvalues. However, selecting the appropriate threshold is a manual task, as it may vary depending on the architecture, dataset, or training algorithms.

To tackle this issue, one possible solution is to retain only the top  $\min\{p,K-1\}$ eigenvalues, which is the maximum rank of $\Sigma_B$. Nevertheless, $\Sigma_B$ can still possess small trailing nonzero eigenvalues.  For example, in the experiments illustrated in Figure~\ref{fig: svals_rn50_basic}, 
the $999$-th eigenvalue is about $2\times 10^{-3}$, significantly smaller than the typical scale of nonzero eigenvalues. Including such small eigenvalues in the computation would yield a substantially large fuzziness value.

To address the numerical stability issue, an alternative approach is to discard the $\Sigma_B$ and instead employ the more well-behaved overall covariance matrix $\Sigma_T$.
As shown in Figure~\ref{fig: svals_rn50_basic}, the eigenvalues of $\Sigma_T$ exhibit a larger scale and a more uniform distribution compared with eigenvalues of $\Sigma_B$, making it a numerically stable choice for pseudoinverse operation.
Interestingly, the quantity $\Sigma_T^\dagger$ naturally emerges in the solution of a loss minimization problem, which we will explore in the next section.

\begin{figure*}[t]
     \begin{subfigure}
         \centering
         \includegraphics[width=0.24\textwidth]{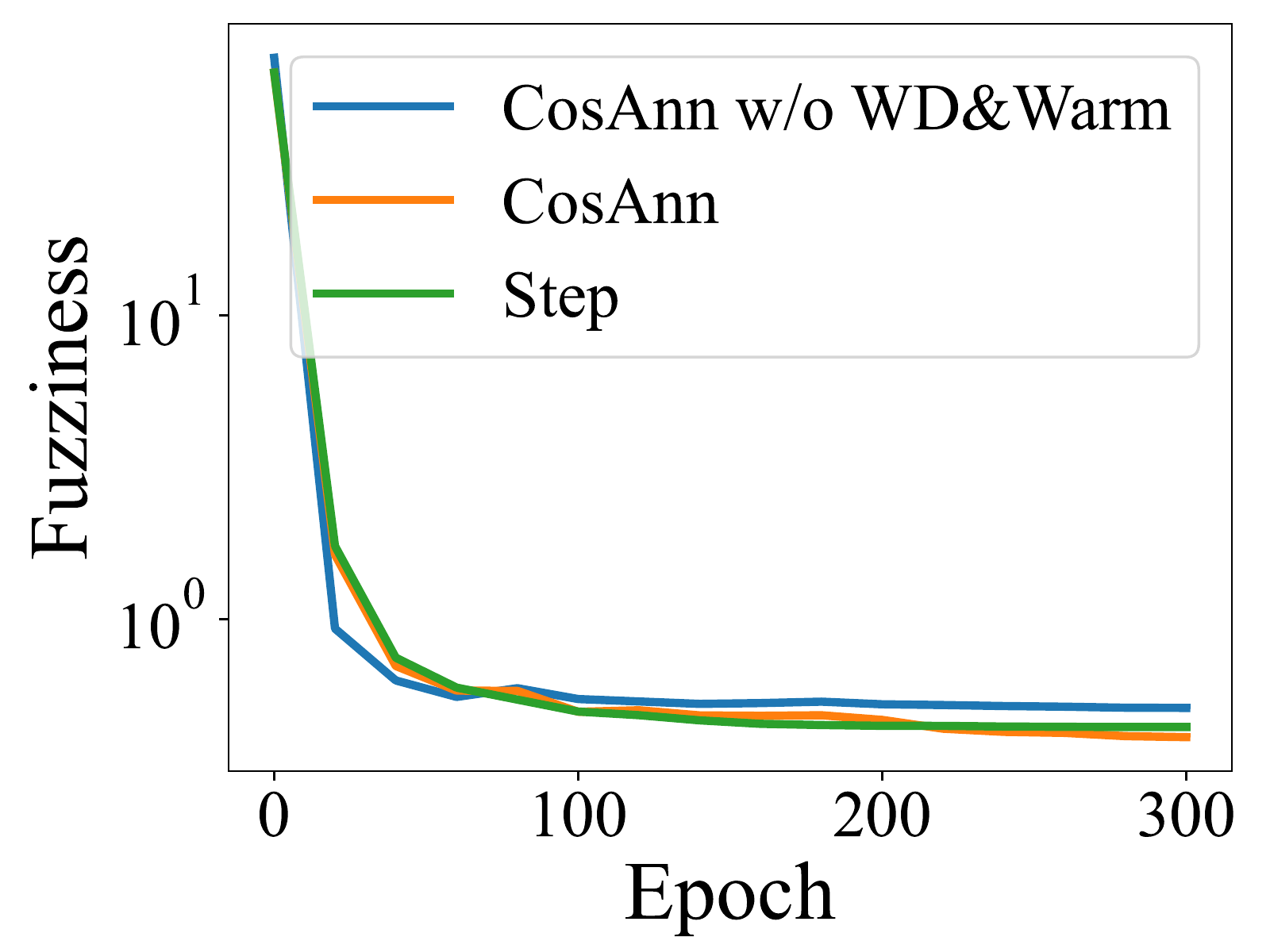}
     \end{subfigure}
     \begin{subfigure}
         \centering
         \includegraphics[width=0.24\textwidth]{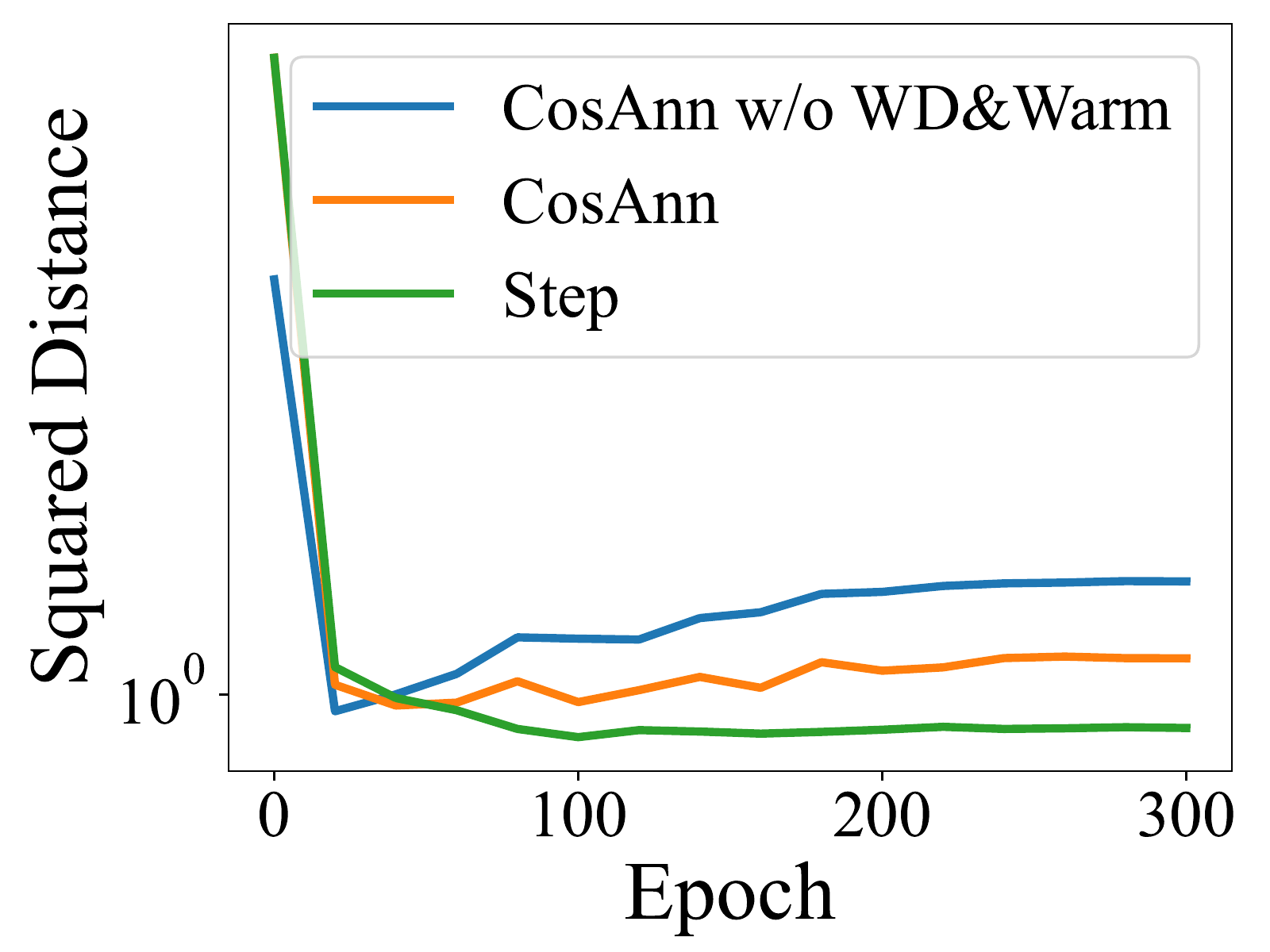}
     \end{subfigure}
     \begin{subfigure}
         \centering
         \includegraphics[width=0.24\textwidth]{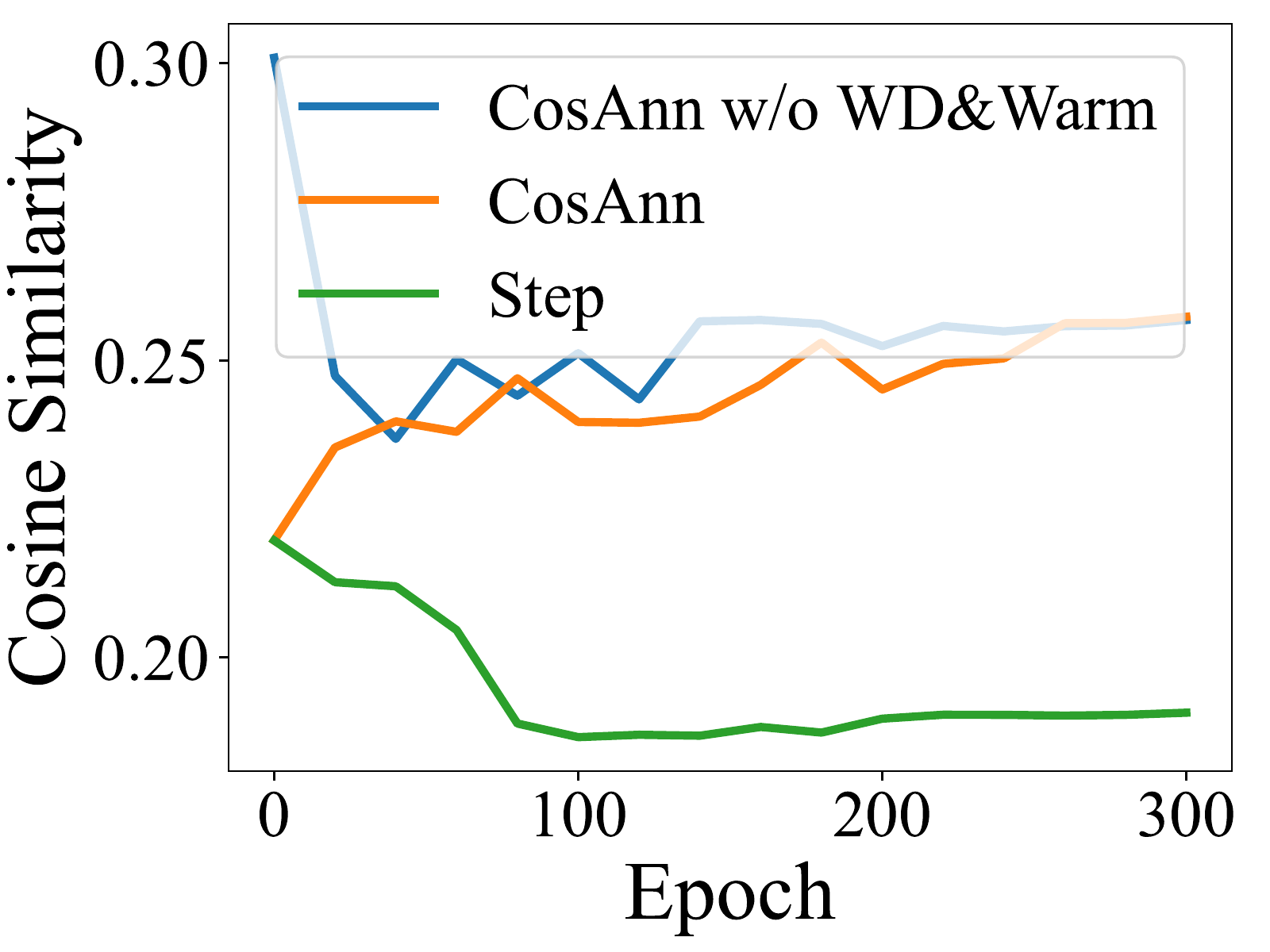}
     \end{subfigure}
     \begin{subfigure}
         \centering
         \includegraphics[width=0.24\textwidth]{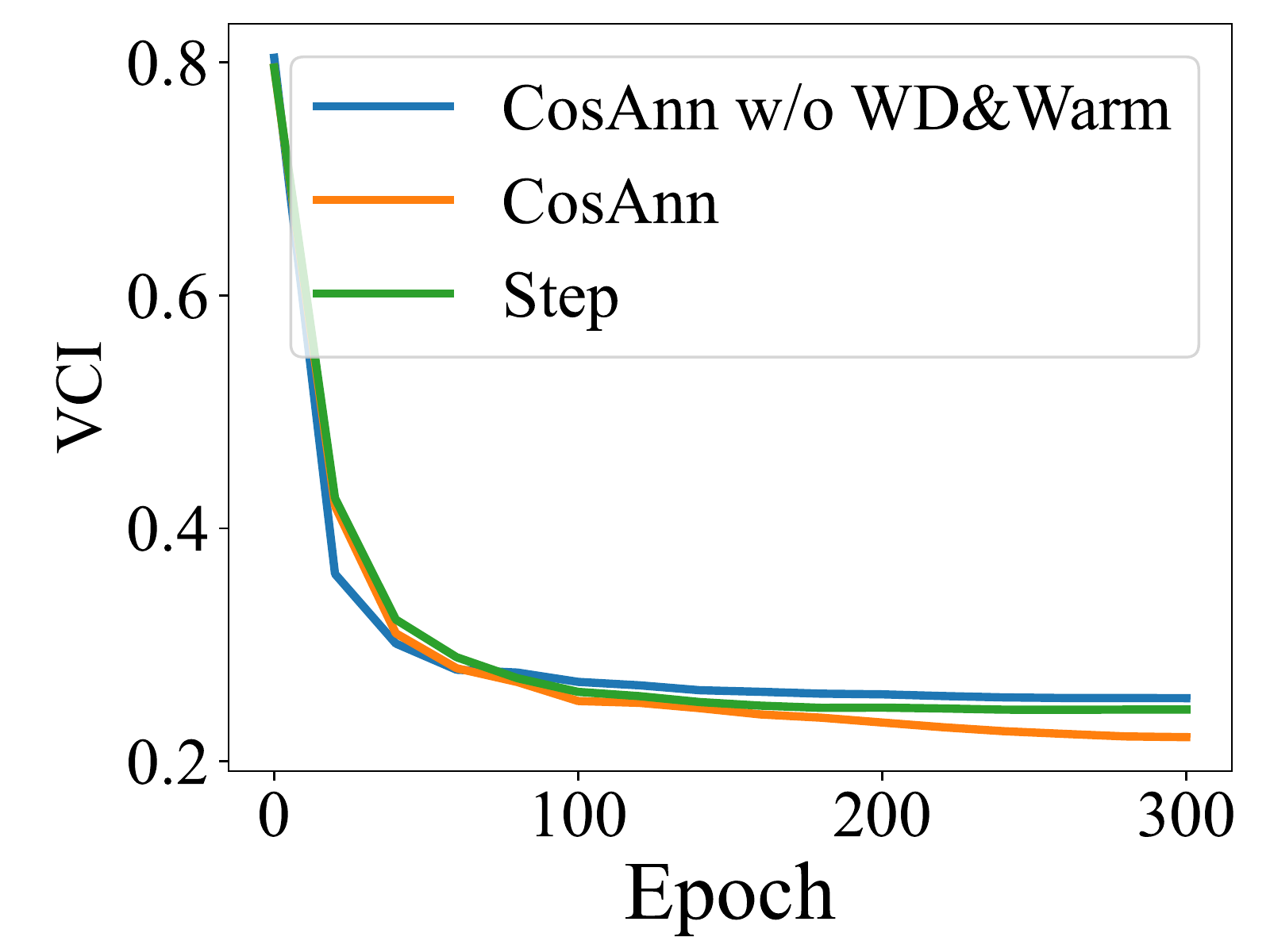}
     \end{subfigure}
        \caption{\textbf{Variability Collapse metrics of training ResNet18 on CIFAR-10 dataset.} From left to right: Fuzziness, Squared Distance, Cosine Similarity and our proposed VCI. The three curves are obtained with different training settings specified below, all achieving $\ge$ 92.1$\%$ test accuracy. \textbf{Green:} step-wise lerning rate decay schedule. \textbf{Orange:} cosine annealing schedule. \textbf{Blue:} cosine annealing schedule without weight decay and warmup.}
        \label{fig: rn18_col}
\end{figure*}

\begin{figure*}[t]
     \begin{subfigure}
         \centering
         \includegraphics[width=0.24\textwidth]{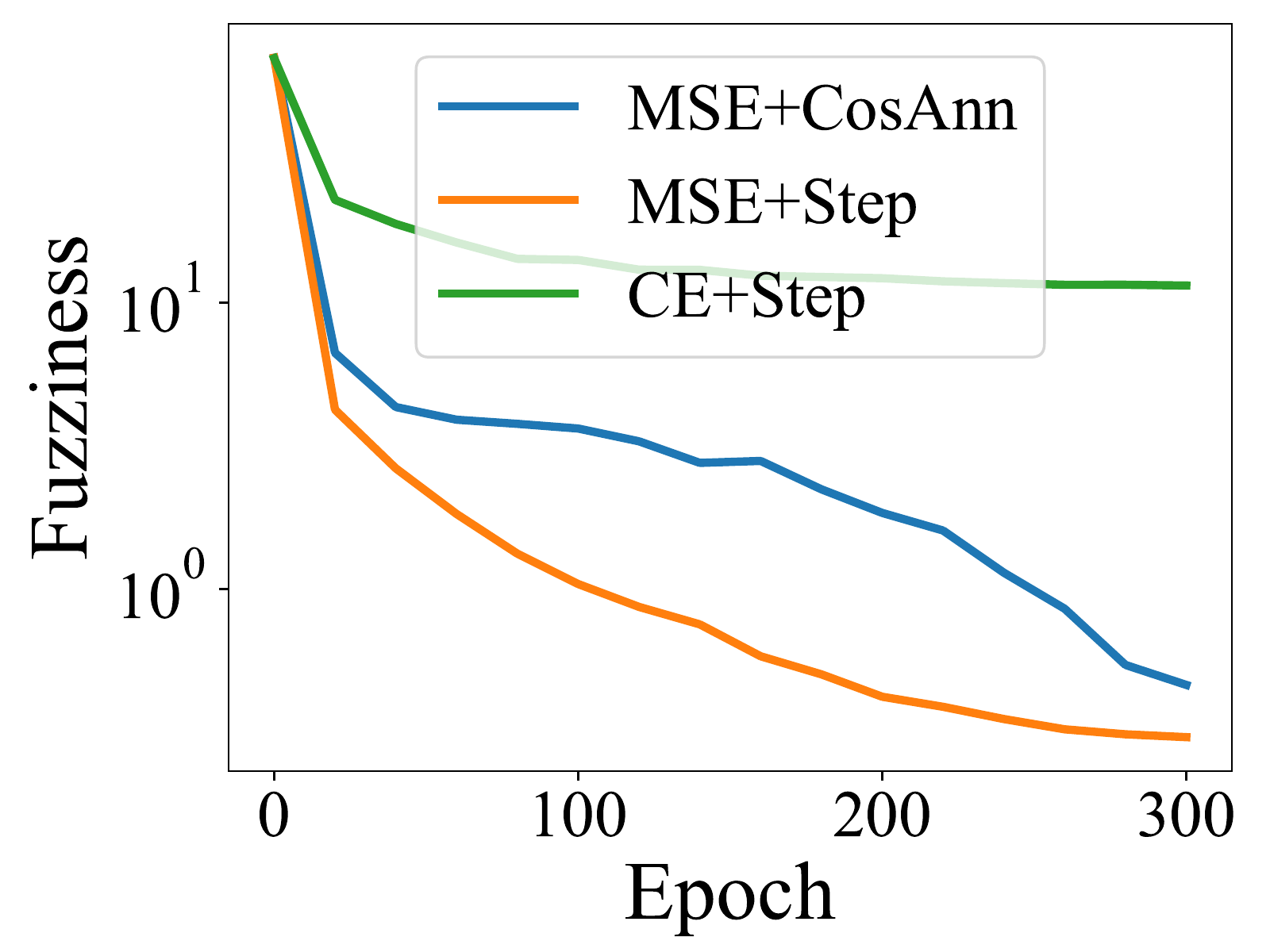}
     \end{subfigure}
     \hfill
     \begin{subfigure}
         \centering
         \includegraphics[width=0.24\textwidth]{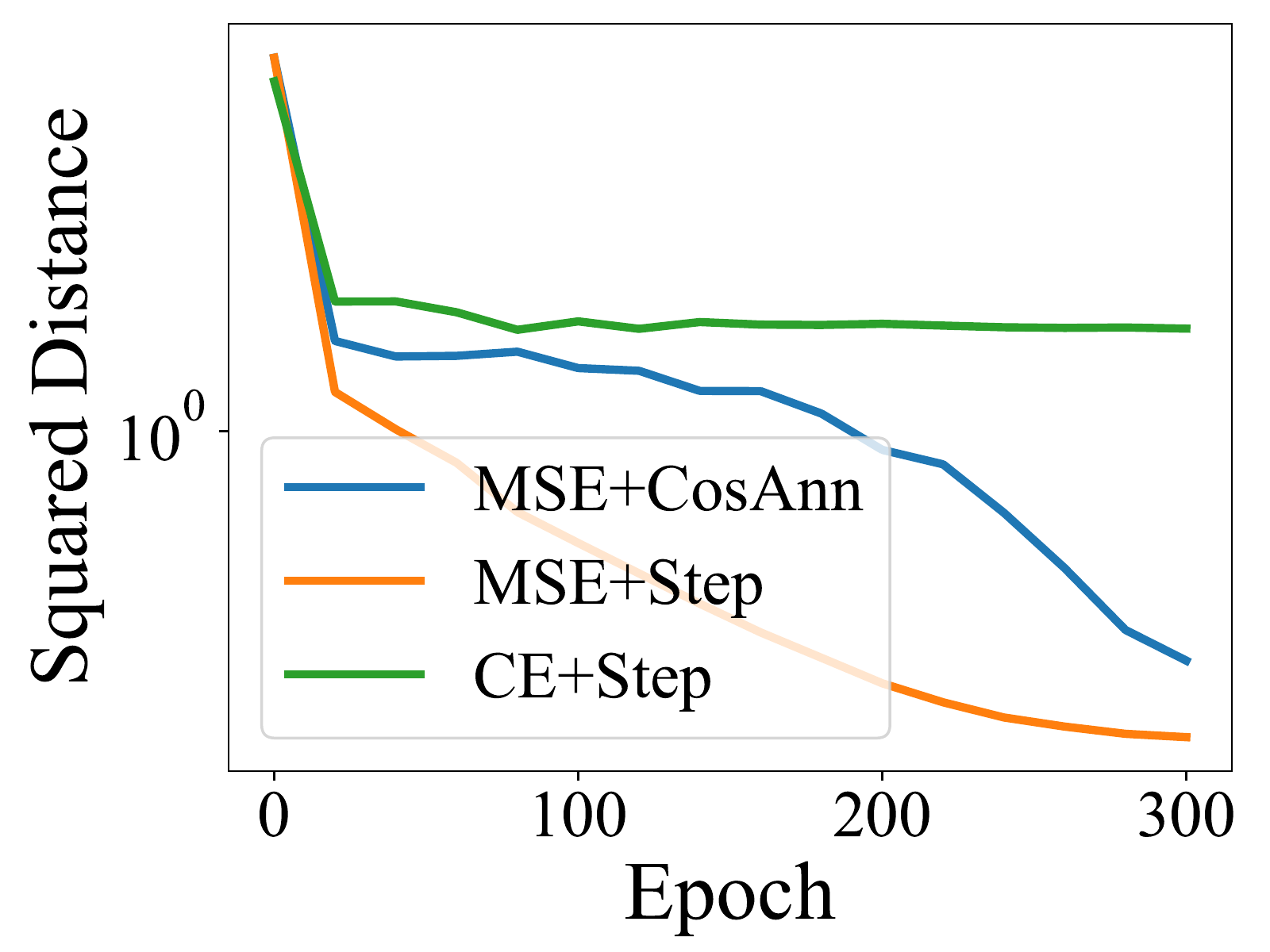}
     \end{subfigure}
     \hfill
     \begin{subfigure}
         \centering
         \includegraphics[width=0.24\textwidth]{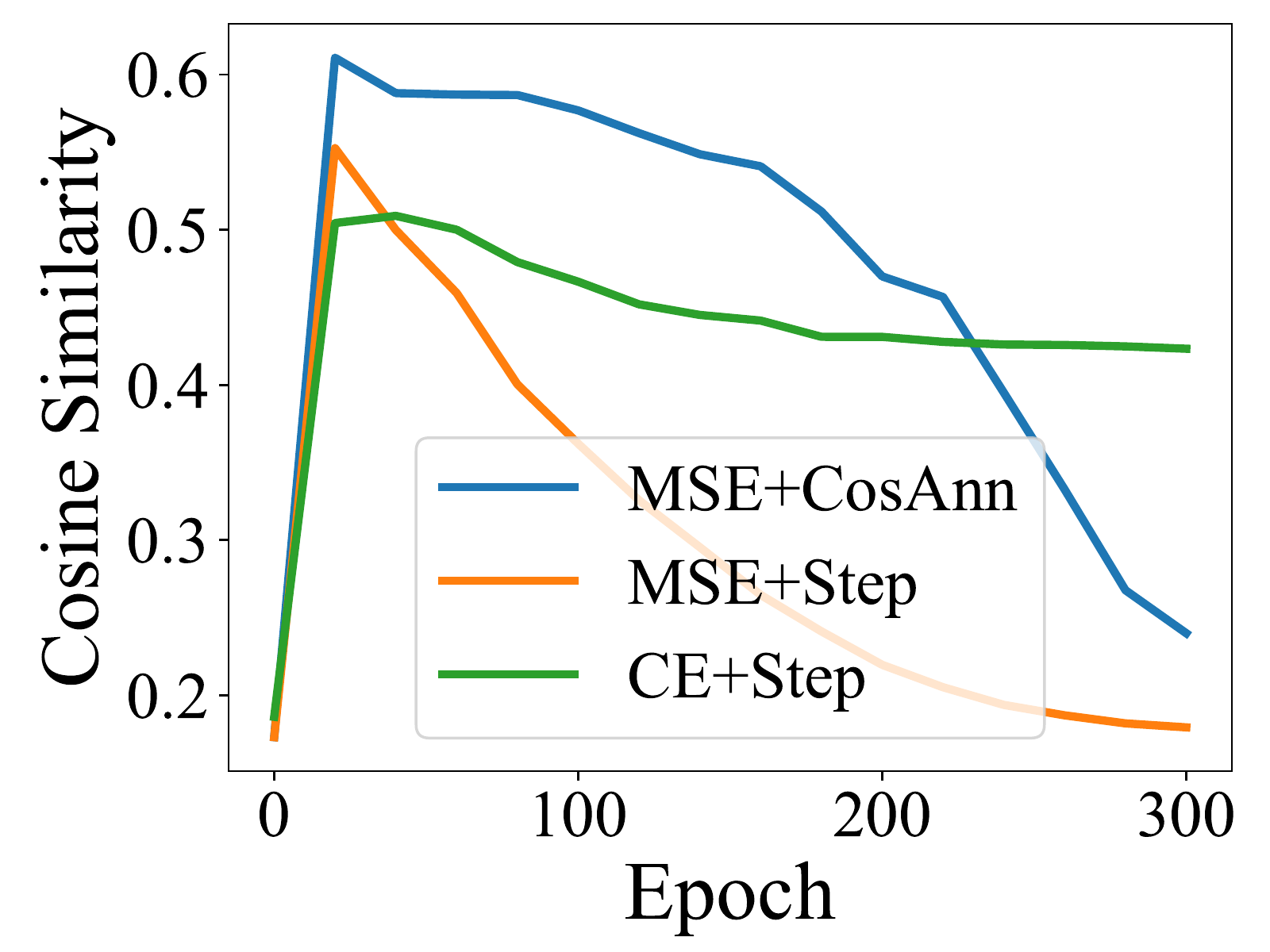}
     \end{subfigure}
     \hfill
     \begin{subfigure}
         \centering
         \includegraphics[width=0.24\textwidth]{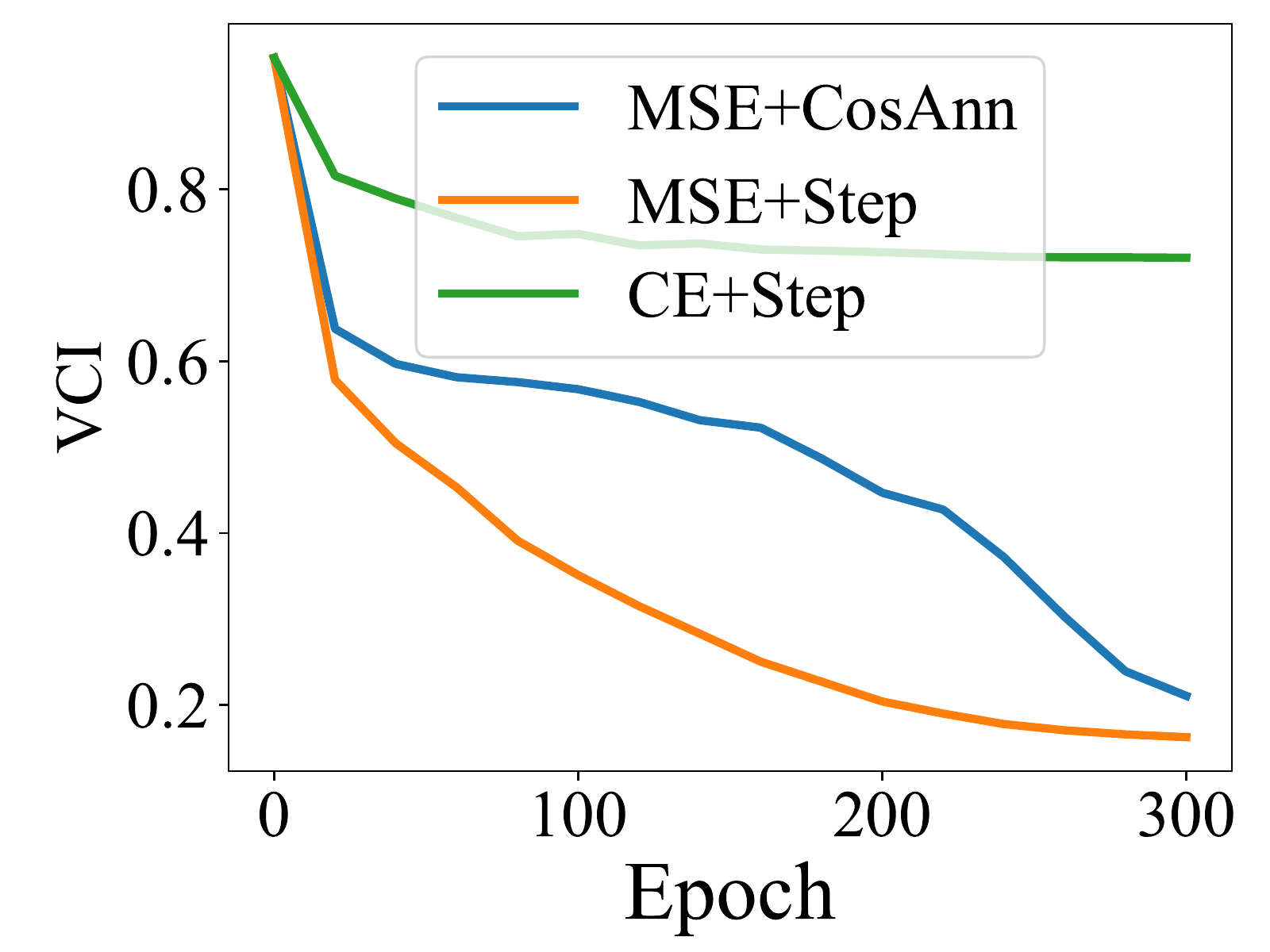}
     \end{subfigure}
        \caption{\textbf{Variability Collapse metrics of training ResNet50 on ImageNet-1k dataset.} From left to right: Fuzziness, Squared Distance, Cosine Similarity and our proposed VCI. The three curves are obtained with different training settings, all achieving $\ge$ 77.8$\%$ test accuracy. green: CE loss. orange: MSE loss. blue: MSE loss + cosine annealing schedule.}
        \label{fig: rn50_col}
\end{figure*}

\section{The Proposed Metric}\label{sec: proposed metric}
As we have discussed, the existing collapse metrics discussed in Section~\ref{sec: previous metrics} do not have the desired properties to fully measure the quality of the representation in downstream tasks. In this section, we introduce a novel and well-motivated collapse metric, which we call Variability Collapse Index~(VCI), that satisfy all the aforementioned properties.



Previous studies~\citep{zhu2021geometric, tirer2022extended} indicate that fully collapsed last layer features minimizes the linear probing loss. Therefore, it is natural to explore the inverse direction, namely, using the linear probing loss to quantify the collapse level of last layer features.  

Suppose we have a labeled dataset with corresponding last layer feature $H=(h_{k,i})_{k\in[K],i\in[N]}$.
We perform linear regression on the last layer to find the optimal parameter $W$ that minimizes the following MSE loss:
\begin{align*}
    L(W,b,H)=\frac{1}{2KN}\sum_{k\in[K],i\in[N]}\|Wh_{k,i}+b-e_k\|^2.
\end{align*}

The following theorem gives the optimal linear probing loss.

\begin{theorem}\label{thm: mse}
The optimal linear probing loss has the following form.
\begin{align*}
    \min_{W,b}L(W,b,H)=-\frac{1}{2K}\Tr\left[\Sigma_{T}^\dagger \Sigma_{B}\right]+\frac{1}{2}-\frac{1}{2K},
\end{align*}
where $\Sigma_B$ and $\Sigma_T$ are the between-class and overall covariance matrix defined in Equation~\ref{eq: between} and~\ref{eq: total}.
\end{theorem}

Theorem~\ref{thm: mse} shows that the information of the minimum MSE loss can be fully captured by the simple quantity $\Tr\left[\Sigma_{T}^\dagger \Sigma_{B}\right]$.
It is easy to see that the minimum of $\Tr[\Sigma_{T}^\dagger \Sigma_{B}]$ is $0$. The following theorem gives an upper bound of $\Tr[\Sigma_{T}^\dagger \Sigma_{B}]$.

\begin{theorem}\label{thm: upper bound}
$\Tr[\Sigma_{T}^\dagger \Sigma_{B}]\le \rank(\Sigma_{B
})$. The equality holds for fully collapsed configuration $\Sigma_W=\mathbf{0}$.
\end{theorem}

The term $\Tr[\Sigma_{T}^\dagger \Sigma_{B}]$ has a positive correlation with the level of collapse in the representation. Theorem~\ref{thm: mse} implies that for MSE loss, a more collapsed representation leads to a smaller loss. Therefore, this term is a natural candidate for collapse metric. 
\begin{definition}
Define the \textbf{Variability Collapse Index~(VCI)} of a set of features $H=(h_{k,i})_{k\in[K],i\in[N]}$ as
\begin{align*}
    \text{VCI}=1-\frac{\Tr[\Sigma_{T}^\dagger \Sigma_{B}]}{\rank(\Sigma_B)},
\end{align*}
where $\Sigma_B$ and $\Sigma_T$ are the between-class and overall covariance matrix defined in Equation~\ref{eq: between} and~\ref{eq: total}.
\end{definition}

One of the advantages  of VCI is its invariance to invertible linear transformations, which is inherited from the invariance of the MSE loss.
\begin{corollary}\label{col: invariance}
VCI is invariant to invertible linear transformation of the feature vector, i.e., multiplying each $h_{k,i}$ with an invertible matrix $U\in\BR^{p\times p}$.
\end{corollary}

\begin{proof}
    From Observation~\ref{prop: loss invariance}, we know that the minimum of the loss function $L(W,b,H)$ is invariant to invertible linear transformations on $H$. This implies the same invariance property of the term $\Tr[\Sigma_T^\dagger \Sigma_B]$. The proof is complete by noting that invertible linear transformation will also preserve the rank of $\Sigma_B$. 
\end{proof}

Another advantage of VCI lies in its numerical stability. This advantage primarily stems from the well-behaved nature of the spectrum of $\Sigma_T$ compared to that of $\Sigma_B$, as discussed in Section~\ref{sec: numerical stability}. Therefore, the pseudo-inverse operation does not lead to an explosive increase in VCI.
Furthermore, one can safely takes $\rank(\Sigma_B)=\min\{p,K-1\}$, since the unknown rank is not the cause of numerically instability as in Fuziness. 






\begin{figure*}[t]
     \begin{subfigure}
         \centering
         \includegraphics[width=0.24\textwidth]{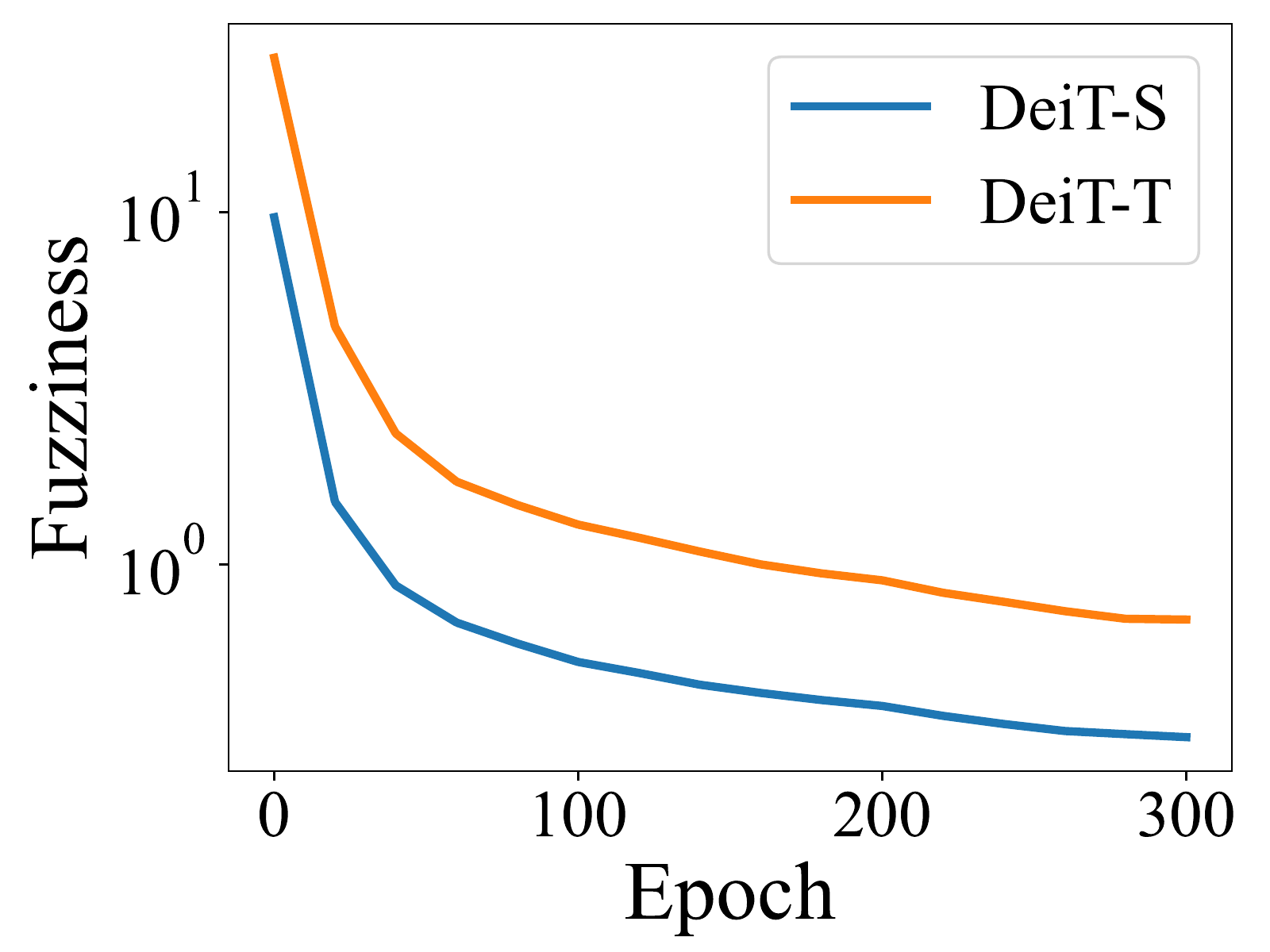}
     \end{subfigure}
     \begin{subfigure}
         \centering
         \includegraphics[width=0.24\textwidth]{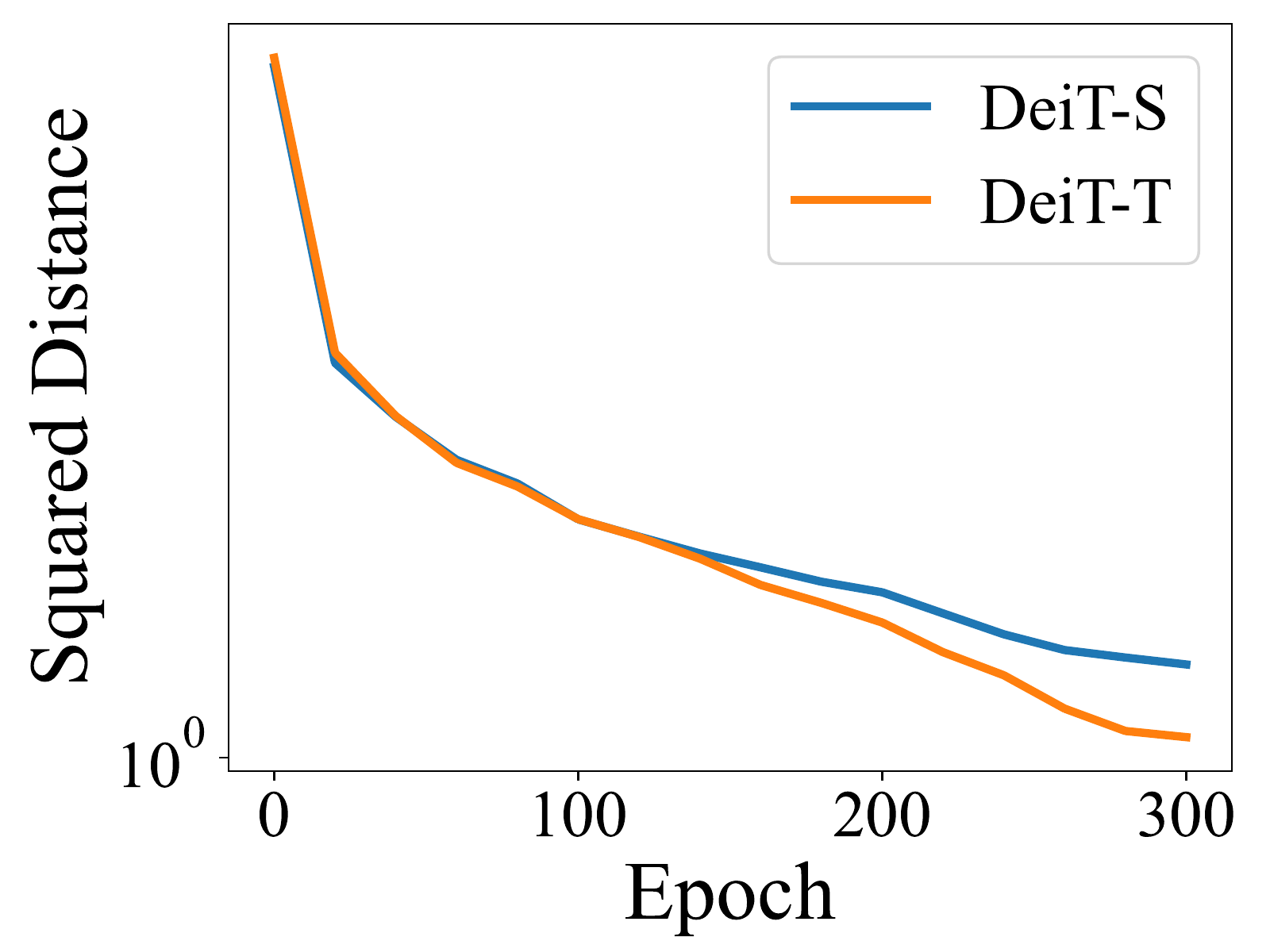}
     \end{subfigure}
     \begin{subfigure}
         \centering
         \includegraphics[width=0.24\textwidth]{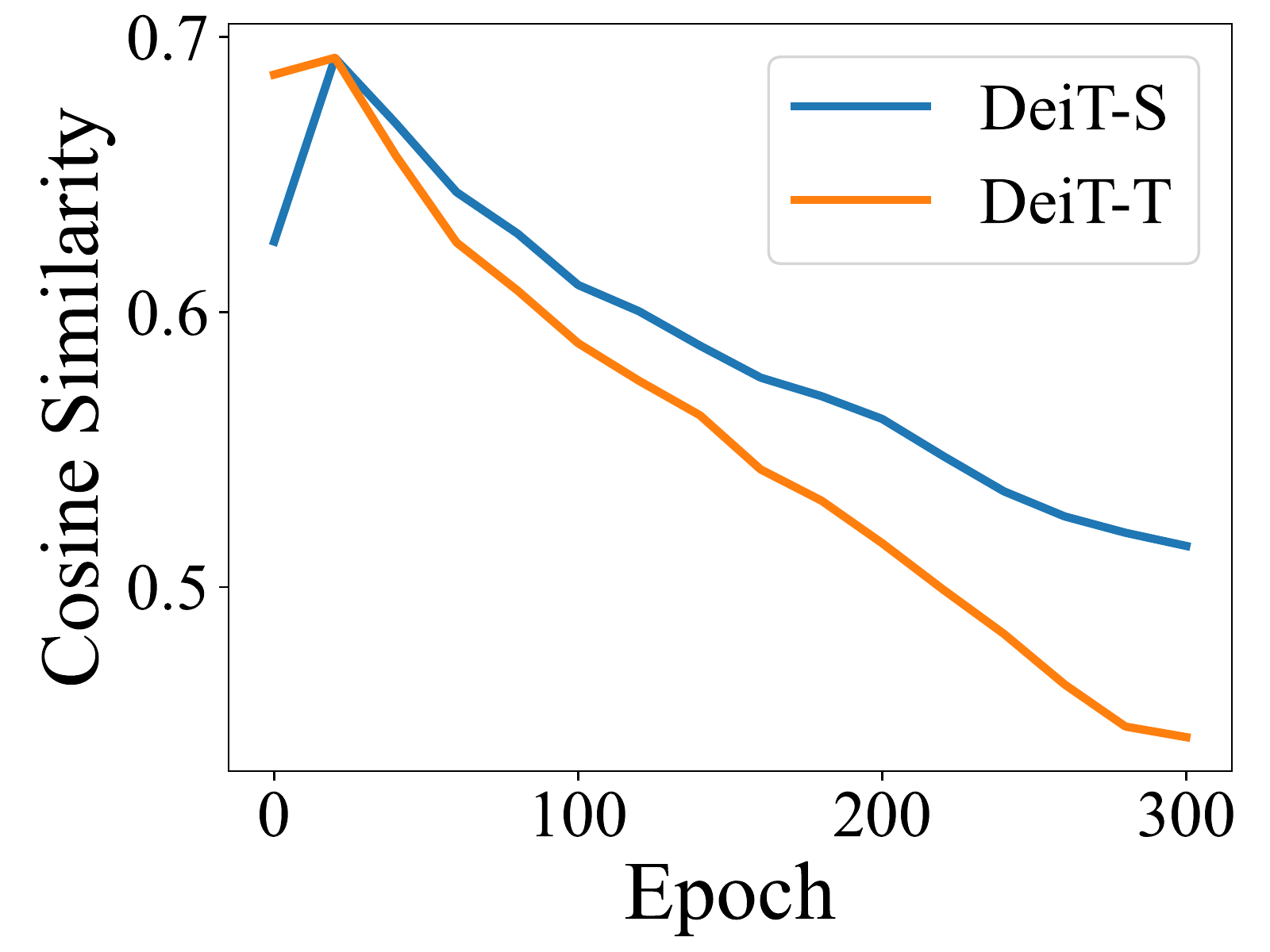}
     \end{subfigure}
     \begin{subfigure}
         \centering
         \includegraphics[width=0.24\textwidth]{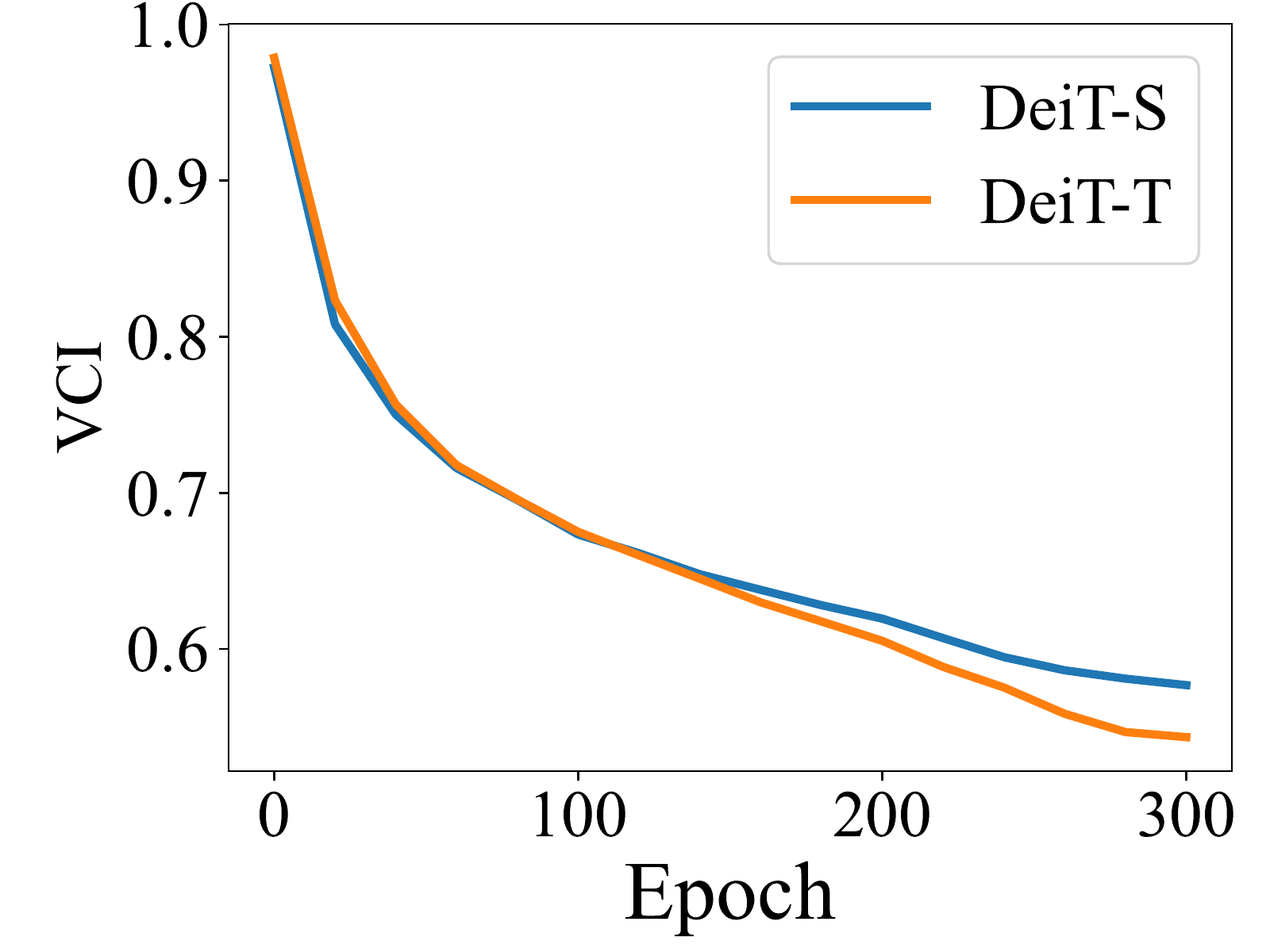}
     \end{subfigure}
        \caption{\textbf{Variability Collapse metrics of training ViT on ImageNet-1k dataset.} From left to right: Fuzziness, Squared Distance, Cosine Similarity and our proposed VCI. \textbf{Blue:} DeiT-S. \textbf{Orange:} DeiT-T. All of the four metrics indicate variability collapse happens for this setting.}
        \label{fig: deit_col}
\end{figure*}

\begin{figure*}[t]
    \centering
     \begin{subfigure}
         \centering
         \includegraphics[width=0.24\textwidth]{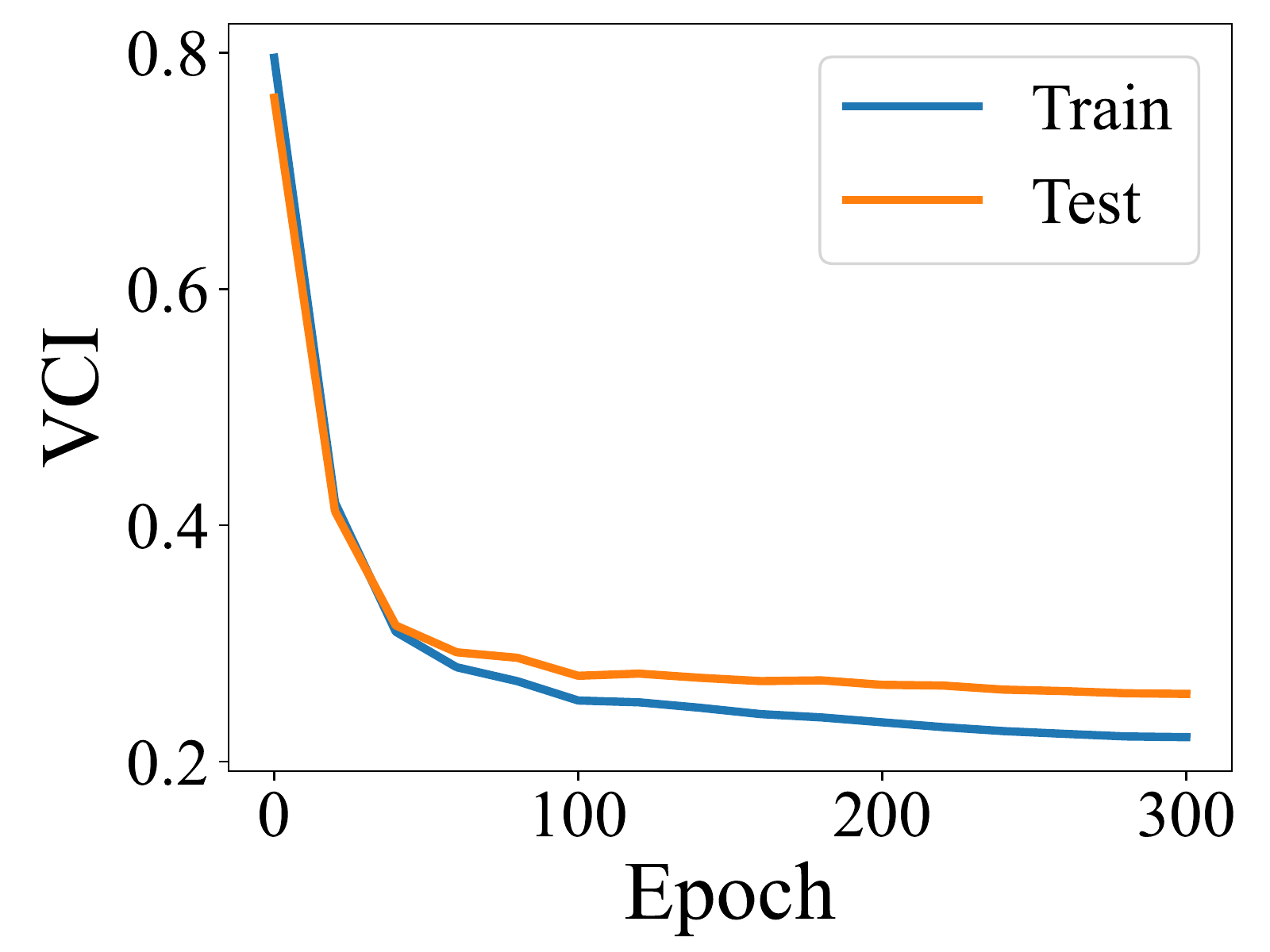}
     \end{subfigure}
     \begin{subfigure}
         \centering
         \includegraphics[width=0.24\textwidth]{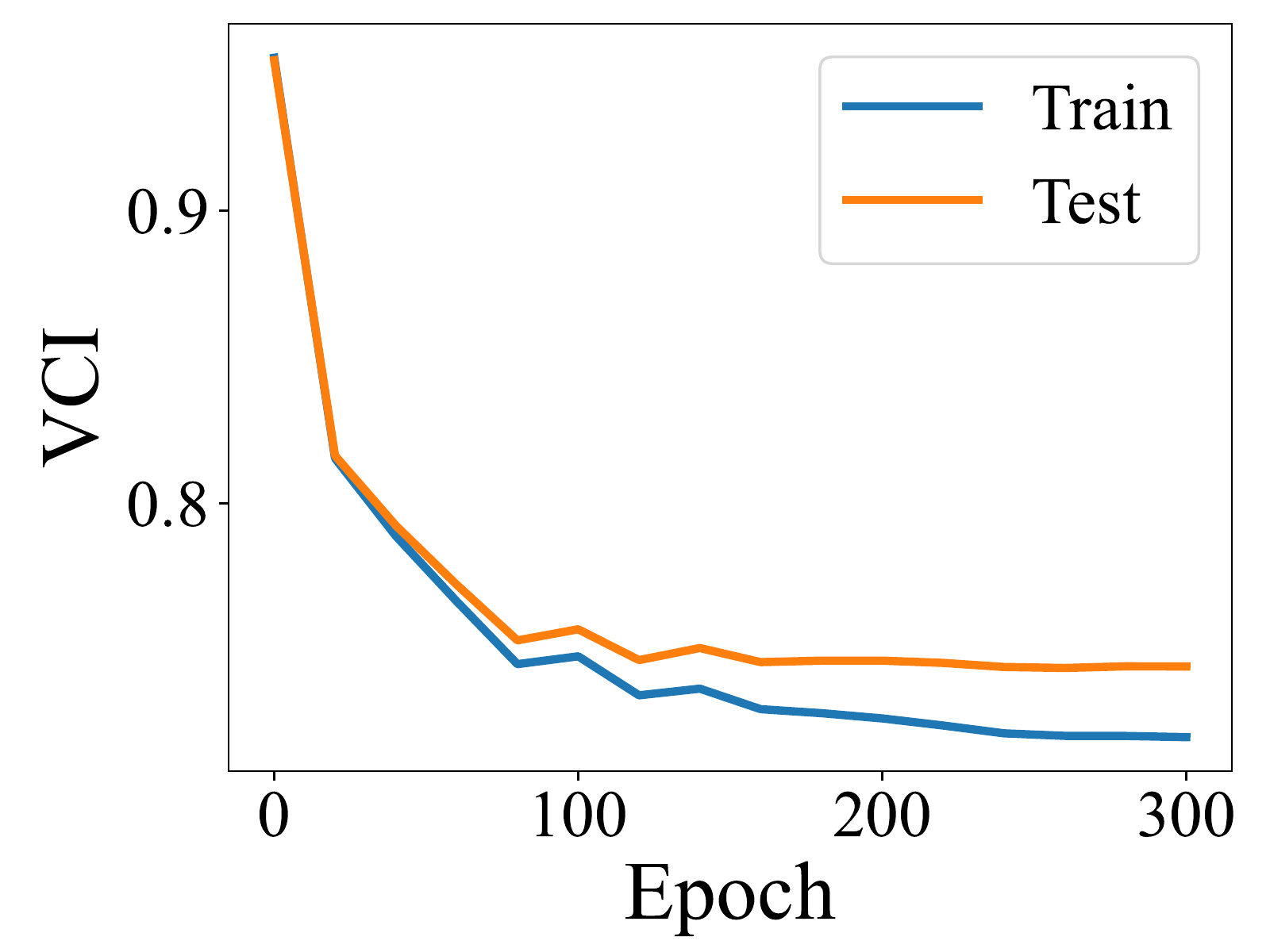}
     \end{subfigure}
     \begin{subfigure}
         \centering
         \includegraphics[width=0.24\textwidth]{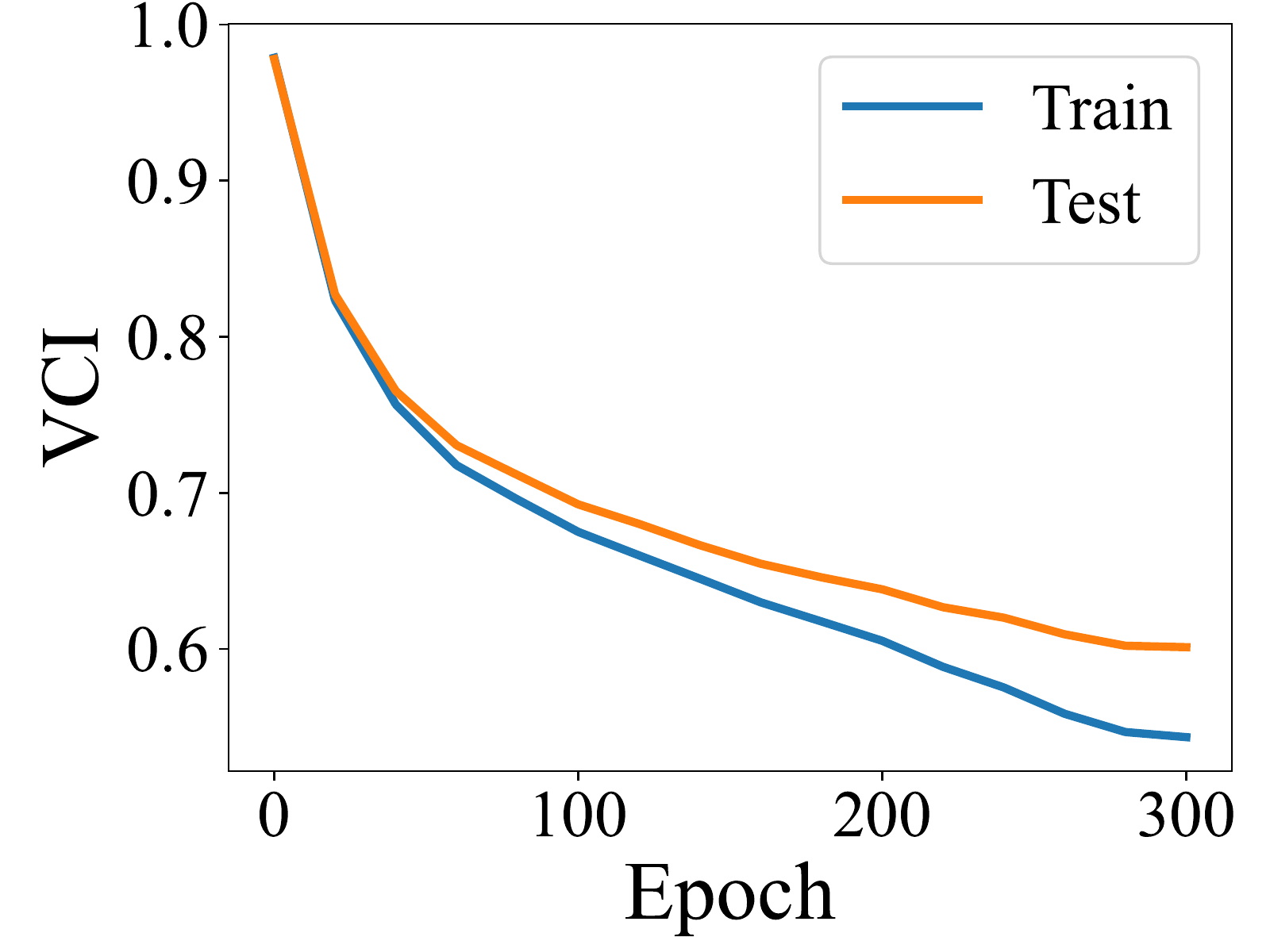}
     \end{subfigure}
        \caption{\textbf{Train Collapse and Test Collapse both happen for VCI.} Train collapse is evaluated on a 50000 subset of ImageNet-1K training dataset. Test collapse is evaluated on the full ImageNet-1K test dataset. \textbf{Left:} ResNet18 on CIFAR-10. \textbf{Middle:} ResNet50 on ImageNet-1K. \textbf{Right:} DeiT-S on ImageNet-1K.}
        \label{fig: traintest-vc}
\end{figure*}


\section{Experiment Results}\label{sec: experiments}

In this section, we present experiments that reflect the differences between the previous variability collapse metrics and our proposed VCI metric.

\begin{figure*}[t]
     \begin{subfigure}
         \centering
         \includegraphics[width=0.24\textwidth]{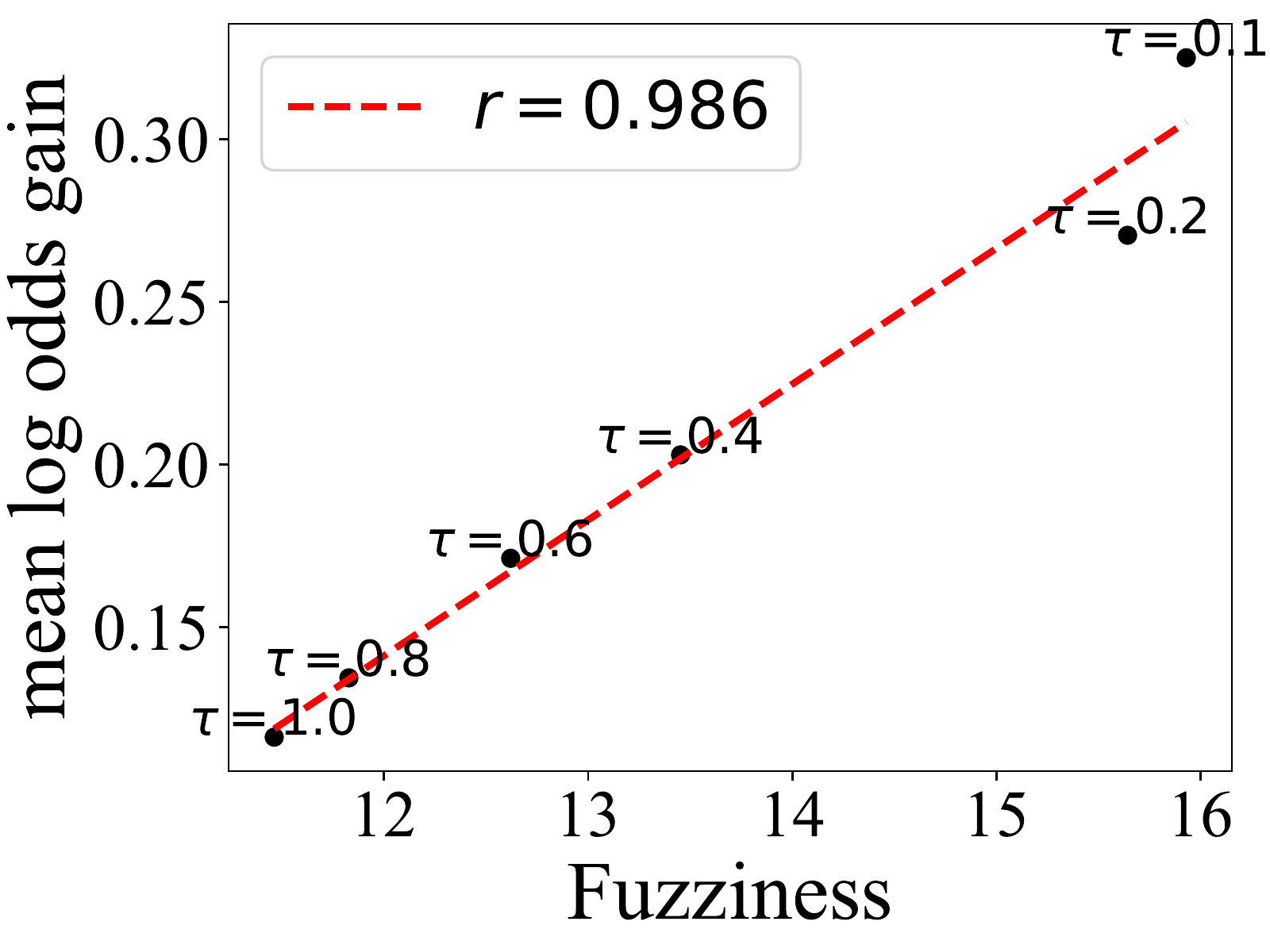}
     \end{subfigure}
     \begin{subfigure}
         \centering
         \includegraphics[width=0.24\textwidth]{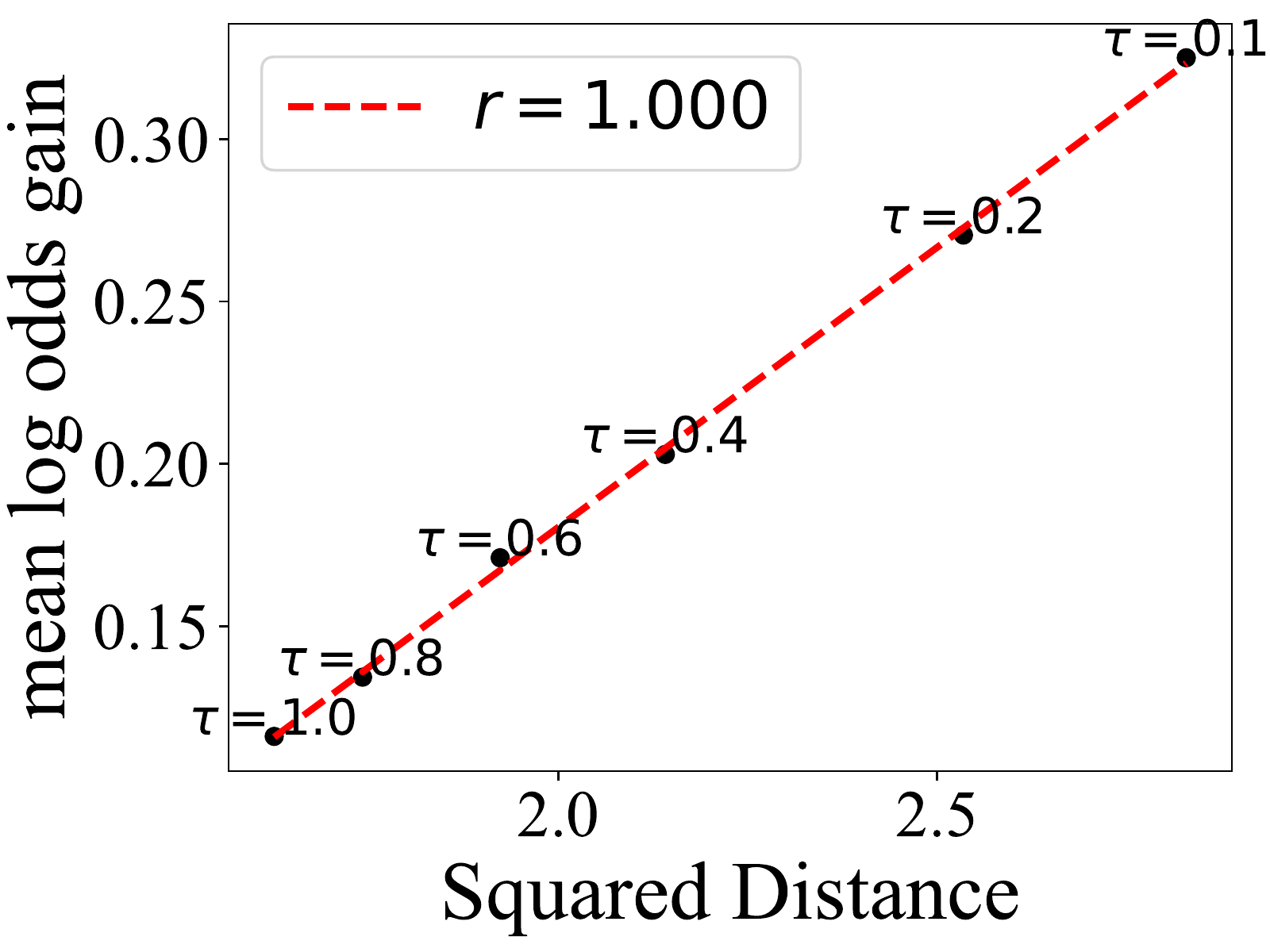}
     \end{subfigure}
     \begin{subfigure}
         \centering
         \includegraphics[width=0.24\textwidth]{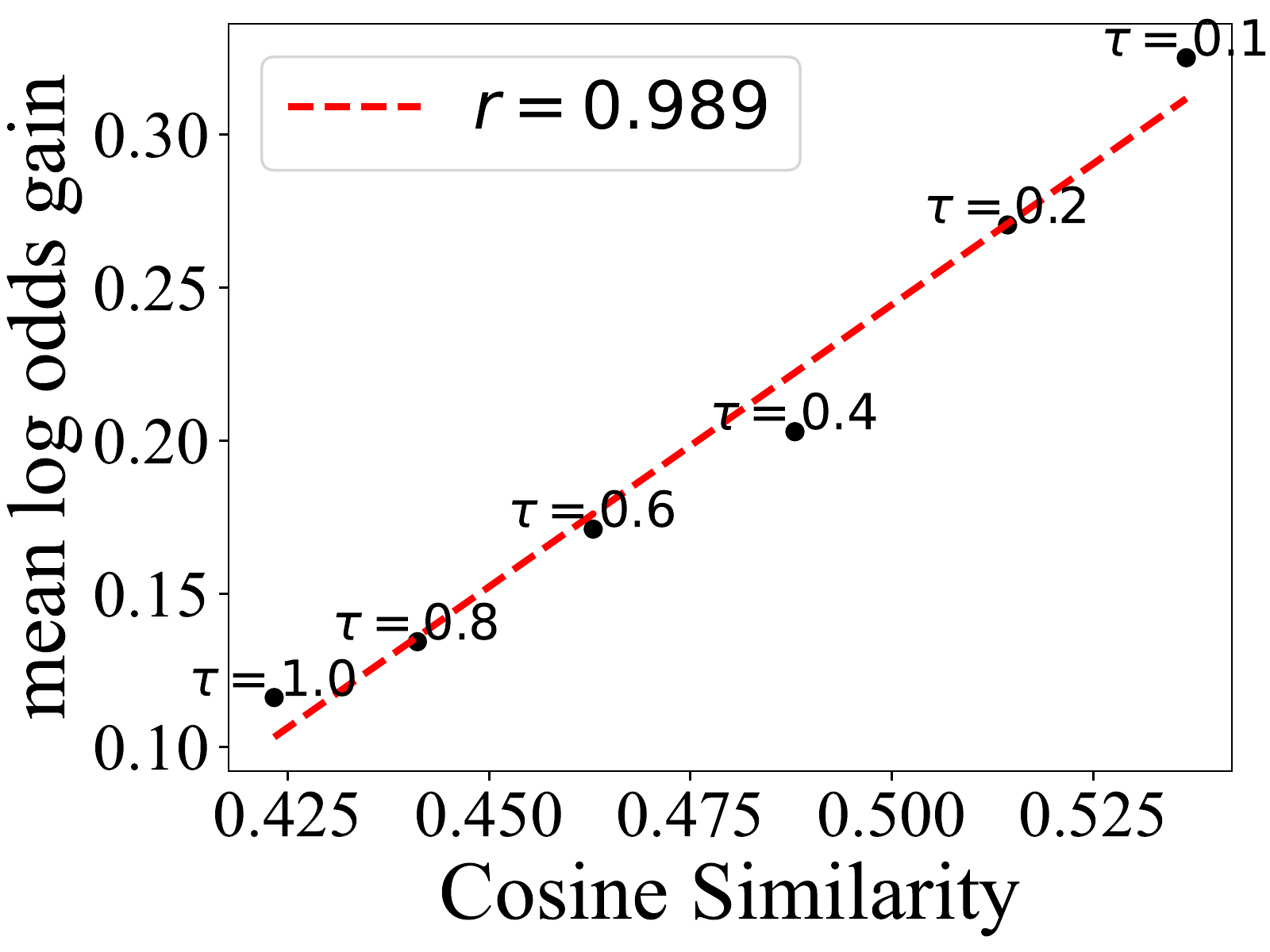}
     \end{subfigure}
     \begin{subfigure}
         \centering
         \includegraphics[width=0.24\textwidth]{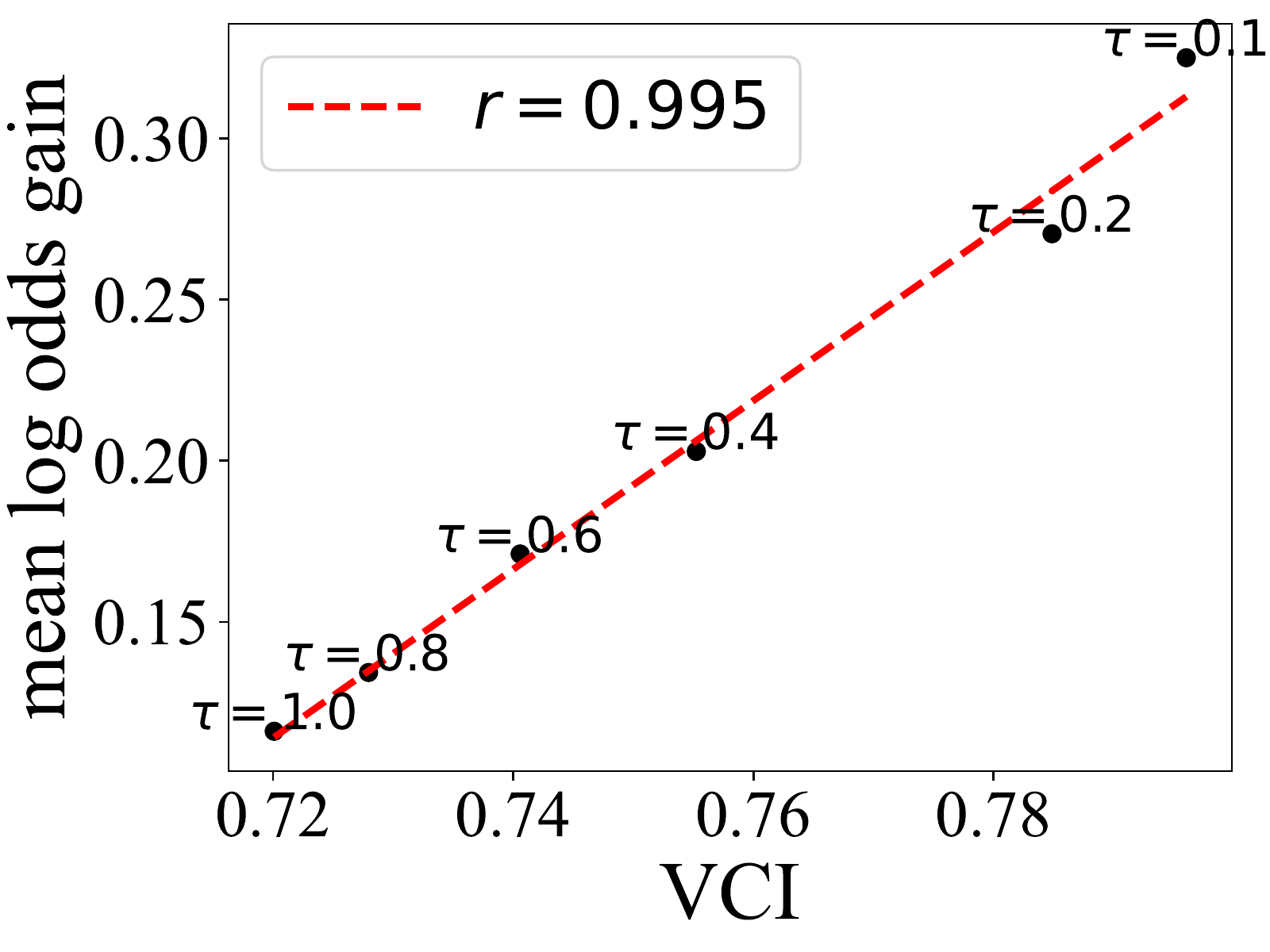}         
     \end{subfigure}
     \\
     \begin{subfigure}
         \centering
         \includegraphics[width=0.24\textwidth]{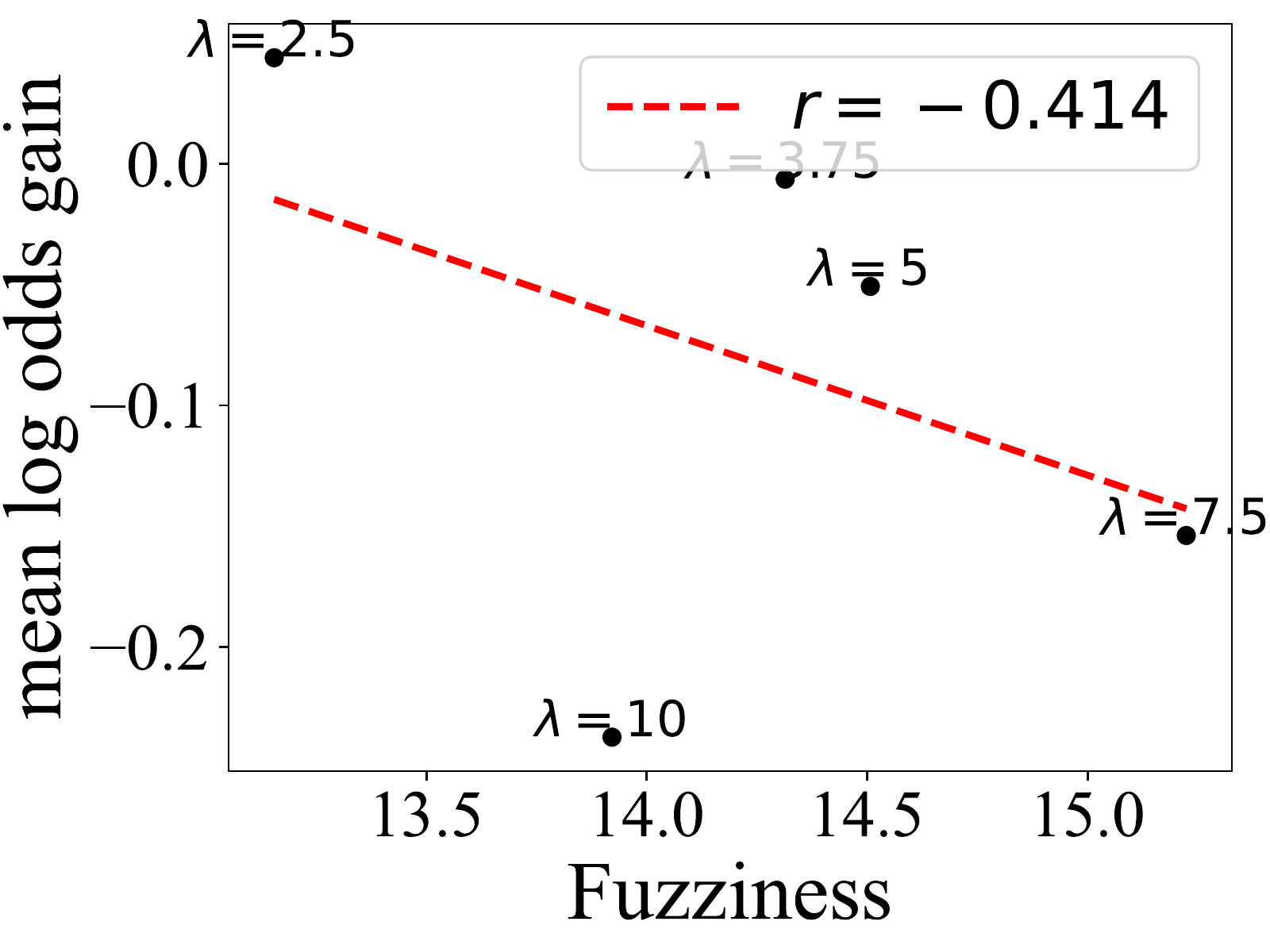}
     \end{subfigure}
     \begin{subfigure}
         \centering
         \includegraphics[width=0.24\textwidth]{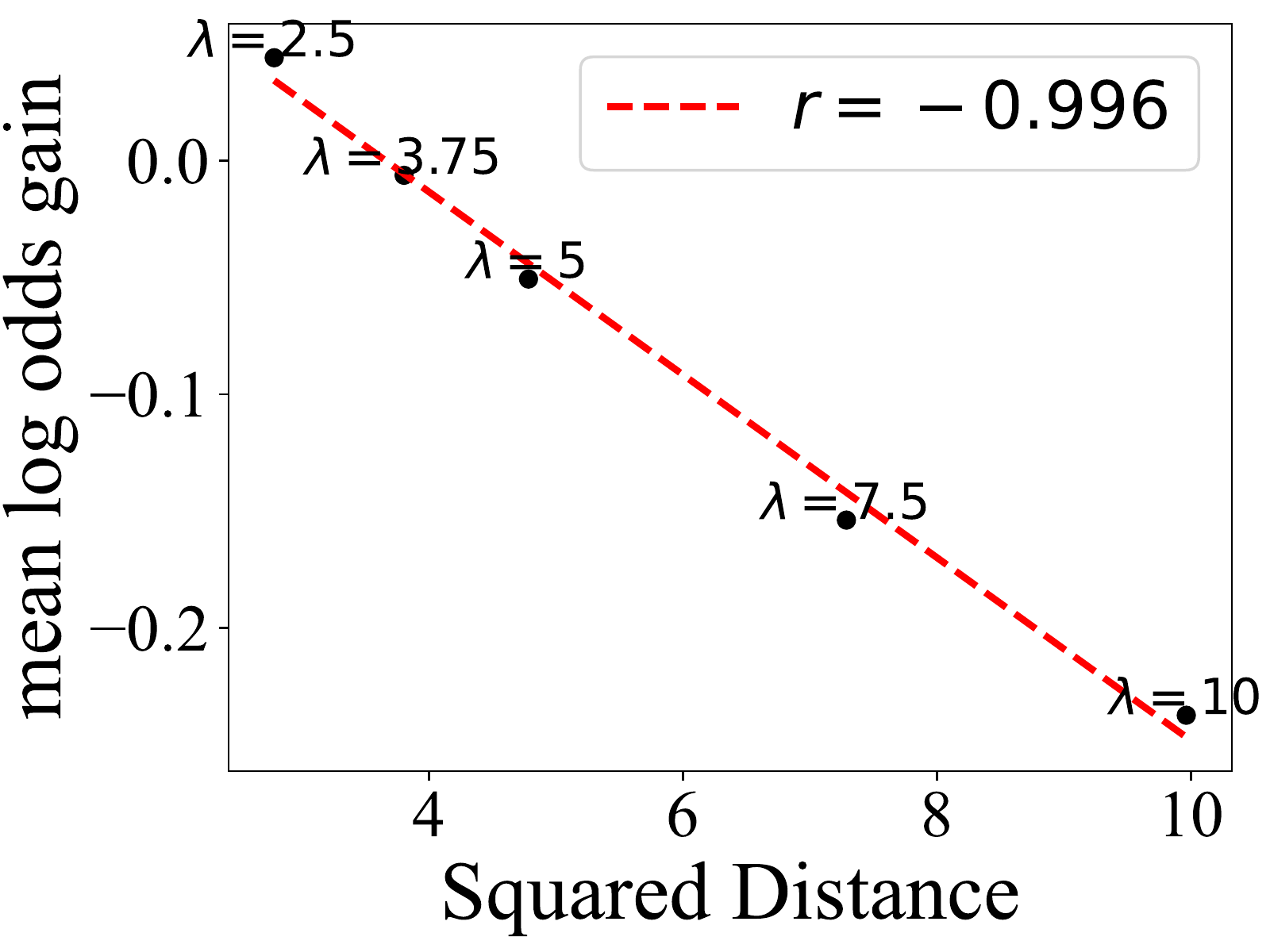}
     \end{subfigure}
     \begin{subfigure}
         \centering
         \includegraphics[width=0.24\textwidth]{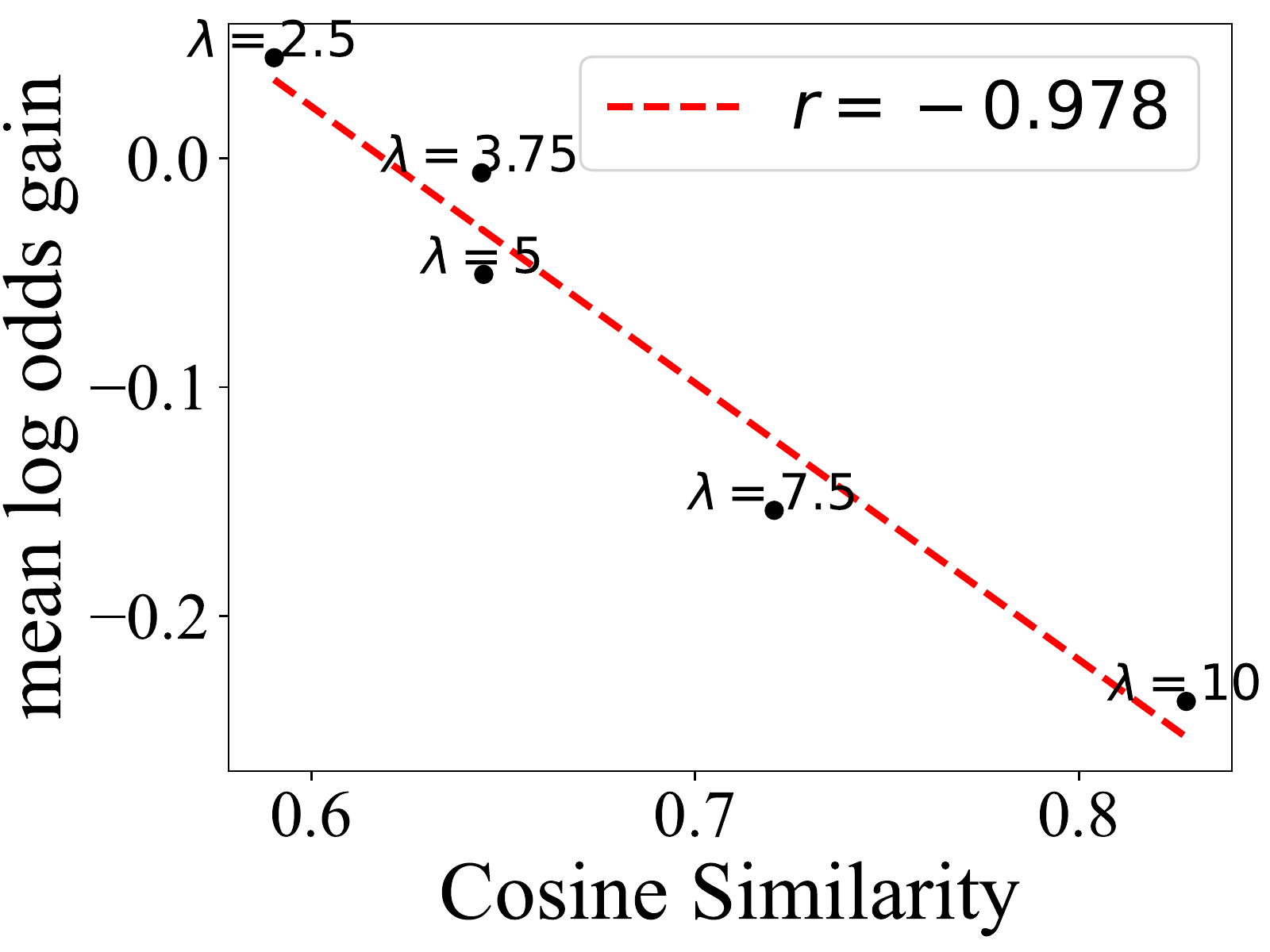}
     \end{subfigure}
     \begin{subfigure}
         \centering
         \includegraphics[width=0.24\textwidth]{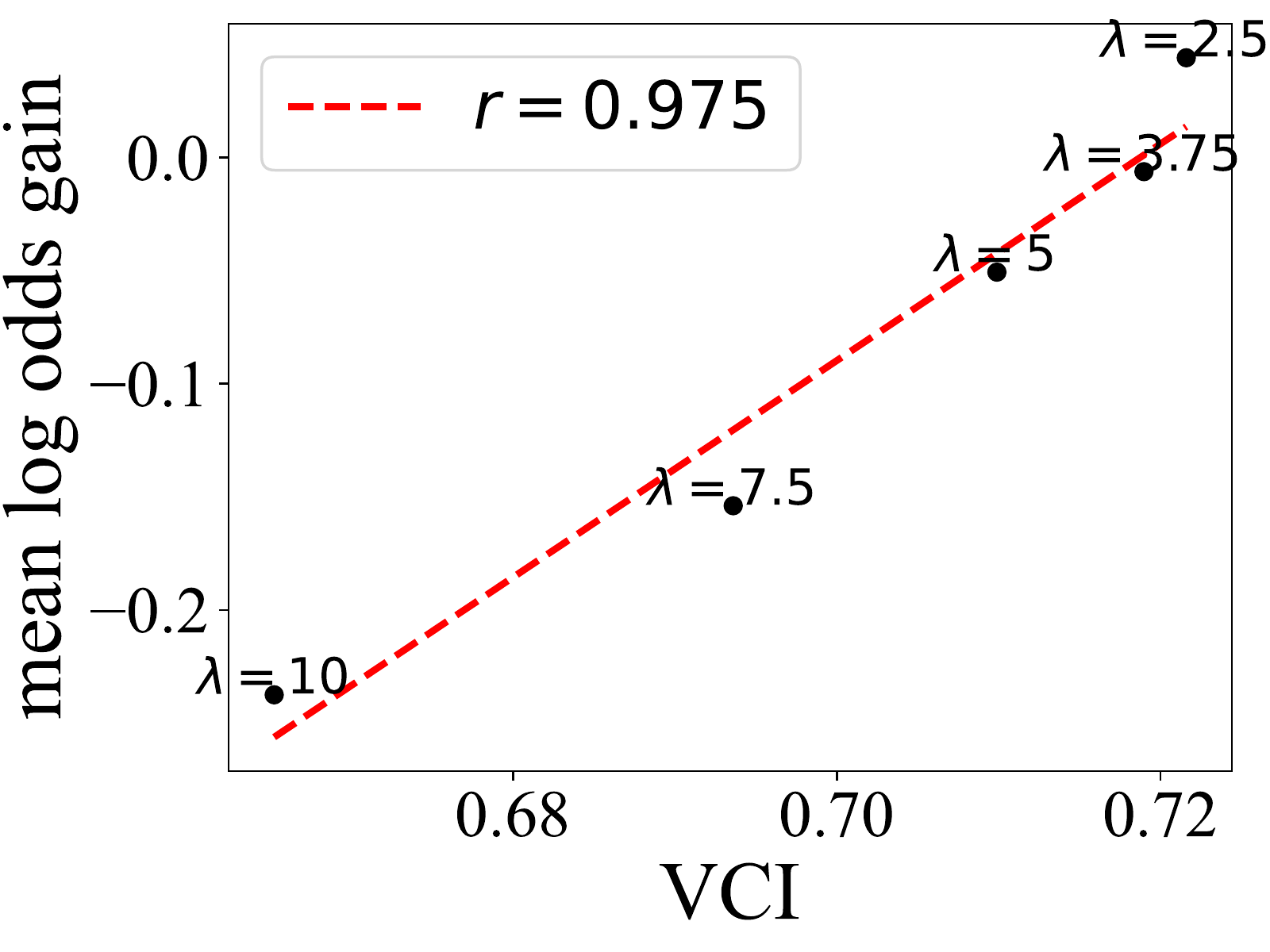}
     \end{subfigure}
        \caption{\textbf{Only VCI consistently indicates transferability in both groups of our experiments:} In each graph, x-axis represents the metric value evaluated on a 50000 subset of ImageNet train set, y-axis shows the mean log odds gain defined as in \cref{eq: mlog}, and the Pearson correlation coefficient is shown in the legend. \textbf{Top Row}: A negative relation between all variability metrics and transferability can be observed when changing the temperature $\tau$ of softmax in pretraining. 
        \textbf{Bottom Row}: Nearly opposite trends emerge on previous variability metrics when we adjust the coefficient $\lambda$ of the Cosine Similarity regularization term. In contrast, VCI maintains a positive correlation with the mean log odds gain.}
        \label{fig: temp_trans}
\end{figure*}

\subsection{Setups}\label{sec: exp_setup}

We conduct experiments to analyze the behavior of four variability collapse metrics, namely Fuzziness, Squared Distance, Cosine Similarity. We evaluate the metrics on the feature layer of ResNet18~\citep{he2016deep} trained on CIFAR10~\citep{krizhevsky2009learning} and ResNet50 / variants of ViT~\citep{dosovitskiy2020image} trained on ImageNet-1K with AutoAugment~\citep{cubuk2018autoaugment} for 300 epochs.
 ResNet18s are trained on one NVIDIA GeForce RTX 3090 GPU, ResNet50s and ViT variants are trained on four GPUs. The batchsize for each GPU is set to 256. The metric values are recorded every 20 epochs, where $\operatorname{rank}(\Sigma_B)$ in the expression of VCI is taken to be $\min \{p, K-1\}$ as stated in the previous section.

For all experiments on ResNet models, We use the implementation of ResNet from the \texttt{torchvision} library, called `ResNet v1.5'. We use SGD with Nesterov Momentum as the optimizer. The maximum learning rate is set to $0.1 \times \text{batch size}/256$. 
We try both the cosine annealing and step-wise learning rate decay scheduler.
When using a step-wise learning rate decay schedule, the learning rate is decayed by a factor of 0.975 every epoch. We also use a linear warmpup procedure of 10 epochs, starting from an initial $10^{-5}$ learning rate. The weight-decay factor is set to $8\times 10^{-5}$. For training on CIFAR10, we replace the random resized crop with random crop after padding 4 pixels on each side as in \citet{he2016deep}. Cross-Entropy loss is used if not specified otherwise.

For DeiT-T and DeitT-S~\citep{touvron2021training}, the two ViT variants used in our experiments, we use AdamW~\citep{Ilya2017decoupled} with a cosine annealing scheduler as the optimizer.
We incorporate a linear warm-up phase of 5 epochs, starting from a learning rate of $10^{-6}$ and gradually increasing to the maximum learning rate of $10^{-3}$.
For other modules of training, such as weight initialization, mixup/cutmix, stochastic depth and random erasing, we keep the same with those of~\citet{touvron2021training}.

At test time for ImageNet-1K, we resize the short side of image to a length of 256 pixels and perform a center crop. When evaluating the variability collapse metrics, we use the same data transformation as at test time. All transformed images are finally normalized with ImageNet mean and standard deviation during training, testing, and metric evaluation.

\subsection{How do Variability Collapse Metrics Evolve as the Training Proceeds}\label{sec: col_exp}
Figure \ref{fig: rn18_col} demonstrates the trend of four different variability collapse metrics when training ResNet18 on CIFAR-10.
It is observed that Squared Distance and Cosine Similarity fail to exhibit a consistent trend of collapse, as explained in \cref{sec: not collapse}.
On the other hand, Fuzziness and VCI show a decreasing trend across these settings.

The results for ResNet50s trained on ImageNet are provided in Figure~\ref{fig: rn50_col}. 
In contrast to the case of ResNet18 on CIFAR10, all evaluated metrics consistently demonstrate a decreasing curve since the ratio of the $V_B^\perp$ part becomes smaller with a smaller $p/K$ value, as shown in Figure \ref{fig: rn50_basic_sqinout}.
Additionally, it is observed that neural networks trained with MSE loss exhibit a higher level of collapse compared to those trained with CE loss, which aligns with the findings of \citet{kornblith2021better}.

The results for ViT variants trained on ImageNet are given in Figure~\ref{fig: deit_col}.
For DeiT-T and DeitT-S with embedding dimensions of 192 and 384, $V_B$ becomes the whole feature space due to $p < K$,  leading to a clearer trend of variability collapse since $V_B^\perp$ becomes $0$.

Finally, we show that test collapse also happens for VCI in \cref{fig: traintest-vc}. 
This indicates that variability collapse is a phenomenon that reflects the properties of underlying data distributions, rather than being solely caused by overfitting the training datasets.
We refer to \cref{sec: addional exp} for comparisons between train collapse and test collapse for other variability metrics.

\subsection{Only VCI consistently Indicates Transferability}\label{sec: trans_exp}

In this section, we investigate the correlation between variability metrics and transferability through two sets of experiments.  We pretrain ResNet50 on ImageNet-1K with a single varying hyperparameter specified within each group. We evaluate the pretrained neural representations using linear probing~\citep{kornblith2019better, chen2020simple} on 10 downstream datasets, including Oxford-IIIT Pets~\citep{parkhi2012cats}, Oxford 102 Flowers~\citep{nilsback2008automated}, FGVC Aircraft~\citep{maji2013fine}, Stanford Cars~\citep{Krause2013CollectingAL}, the Describable Textures Dataset (DTD)~\citep{cimpoi2014describing}, Food-101 dataset~\citep{bossard2014food}, CIFAR-10 and CIFAR-100~\citep{krizhevsky2009learning}, Caltech-101~\citep{feifei2004caltech}, the SUN397 scene dataset~\citep{xiao2010sun}.
We use L-BFGS to train the linear classifier, with the optimal $L_2$-penalty strength determined by searching through 97 logarithmically spaced values between $10^{-6}$ and $10^6$ on a validation set. We provide the raw experiment results in~\cref{sec: raw_data}.

We use the following \newterm{mean log odds gain}
\begin{equation}\label{eq: mlog}
\operatorname{MLOG} = \frac{1}{10} \sum_{i=1}^{10} \log \frac{p_i}{1-p_i} - \log \frac{p_{\text{pretrain}}}{1-p_{\text{pretrain}}}
\end{equation}
to measure the transferability of a neural representation, where $p_{\text{pretrain}}$ is the final test accuracy in pretraining.
Compared with~\citet{kornblith2019better}, we   
subtract the log odds of the pretrain accuracy from the mean log odds of linear classification accuracy over the downstream tasks, to isolate the impact of variability collapse on transfer performance.

In the first group, we change the temperature $\tau$ in the softmax function 
\vskip -20pt
\begin{small}
\begin{align*}
\operatorname{Softmax}_\tau (z) = \left(\frac{\exp(\frac{1}{\tau}z_1)}{\sum_{k=1}^K \exp(\frac{1}{\tau}z_k)}, \cdots, \frac{\exp(\frac{1}{\tau}z_K)}{\sum_{k=1}^K \exp(\frac{1}{\tau}z_k)}\right).
\end{align*}
\end{small}
The results of the first group of experiments are shown in the top row of \cref{fig: temp_trans}. The results are consistent with the findings in \cite{kornblith2021better}, as all considered metrics show a negative relation between variability collapse and transfer performance.

In the second group of experiments, we introduce regularization to control the collapse behavior of neural networks~\citep{kornblith2021better}. The regularization term we used is the average within-class cosine similarity divided by the number of data points of each class in the batch.
By varying the value of $\lambda$ multiplied to the regularization term, we investigate whether the observed correlation in the first group still holds true. The bottom row of \cref{fig: temp_trans} shows that for the three previous metrics, the correlation changes from positive to negative, or vice versa.
However, a strong positive correlation consistently holds between VCI and transferability. 
Therefore, VCI serves as an effective indicator of transfer performance, compared to other variability collapse metrics.

\section{Conclusions and Future Directions}
In this paper, we study the variability collapse phenomenon of neural networks, and propose the VCI metric as a quantitative characterization. We demonstrate that VCI enjoys many desired properties, including invariance and numerical stability, and verify its usefulness via extensive experiments. 

Moving forward, there are several promising directions for future research. 
Firstly, it would be beneficial to explore the applicability of VCI to a broader range of training recipes and architectures, by analyzing its performance using alternative network architectures, training methodologies, and datasets.
Secondly , it would be valuable to conduct theoretical investigations into the relationship between variability collapse and transfer accuracy. Understanding the mechanisms and principles behind this could provide insights to designing better transfer learning algorithms.


\section*{Acknowledgements}

The authors would like to acknowledge the support from the 2030 Innovation Megaprojects of China (Programme on New Generation Artificial Intelligence) under Grant No. 2021AAA0150000.

\bibliography{example_paper}
\bibliographystyle{icml2023}

\newpage
\appendix
\onecolumn

\section{Proofs}

\subsection{Proof of Proposition~\ref{prop: not collapse}}
\label{apx: not collapse}
\begin{proof}
    Since $p>K$, we can find a nonzero vector $v\in\BR^p$, such that $Wv=\mathbf{0}$. For some $\lambda>0$, define the elements of $H^\prime$ as
    $$
    h_{k,i}^\prime = 
    \begin{cases}
    h_{k,i}+\lambda v &  i=1, \\
    h_{k,i} & 2\le i \le N.
    \end{cases}
    $$
    For such an $H^\prime$, we have $Wh_{k,i}=Wh_{k,i}^\prime$, and therefore $L(W,b,H)=L(W,b,H^\prime)$. Furthermore, we can calculate $\mu_k(H^\prime)=\mu_k(H)+\frac{\lambda}{N}v$, and 
    $\mu_G(H^\prime)=\mu_G(H)+\frac{\lambda}{N}v$, which implies that $\Sigma_B(H^\prime)=\Sigma_B(H)$.

    Next, we calculate $\Sigma_W(H^\prime)$:
    \begin{align*}
        \Sigma_W(H^\prime)&=\frac{1}{KN}\sum_{k\in[K], i\in[N]}\left(h_{k,i}^\prime-\mu_k(H^\prime)\right)\left(h_{k,i}^\prime-\mu_k(H^\prime)\right)^\top \\
        &=\frac{1}{KN}\sum_{k=1}^K\left[
        \left(h_{k,1}+\lambda v-\mu_k(H)-\frac{\lambda}{N}v\right)\left(h_{k,1}+\lambda v-\mu_k(H)-\frac{\lambda}{N}v\right)^\top 
        \right.\\& \left.\quad+
        \sum_{i=2}^N\left(h_{k,i}-\mu_k(H)-\frac{\lambda}{N}v\right)\left(h_{k,i}-\mu_k(H)-\frac{\lambda}{N}v\right)^\top\right]
        \\&=\frac{1}{KN}\sum_{k=1}^K\left[
        \left(h_{k,1}-\mu_k(H)\right)\left(h_{k,1}-\mu_k(H)\right)^\top
        +\frac{\lambda(N-1)}{N}v(h_{k,1}-\mu_k(H))^\top
        \right.\\& \left.\quad+\frac{\lambda(N-1)}{N}(h_{k,1}-\mu_k(H))v^\top 
        +\frac{\lambda^2(N-1)^2}{N^2}vv^\top+
        \sum_{i=2}^N\left(h_{k,i}-\mu_k(H)\right)\left(h_{k,i}-\mu_k(H)\right)^\top
        \right.\\& \left.\quad-\frac{\lambda}{N}v\sum_{i=2}^N\left(h_{k,i}-\mu_k(H)\right)^\top-\frac{\lambda}{N}\sum_{i=2}^N\left(h_{k,i}-\mu_k(H)\right)v^\top+\frac{\lambda^2(N-1)}{N^2}vv^\top\right]
        \\&=\frac{1}{KN}\sum_{k=1}^K\left[\sum_{i=1}^N\left(h_{k,i}-\mu_k(H)\right)\left(h_{k,i}-\mu_k(H)\right)^\top+\frac{\lambda}{N}v\left((N-1)h_{k,1}-\sum_{i=2}^N h_{k,i}\right)^\top
        \right.\\& \left.\quad+\frac{\lambda}{N}\left((N-1)h_{k,1}-\sum_{i=2}^N h_{k,i}\right)v^\top+\frac{\lambda^2(N-1)}{N}vv^\top\right]\\
        &=\Sigma_W(H)+\frac{\lambda}{KN^2}\left[v\sum_{k=1}^K\left((N-1)h_{k,1}-\sum_{i=2}^Nh_{k,i}\right)^\top
        \right.\\& \left.\quad+\sum_{k=1}^K\left((N-1)h_{k,1}-\sum_{i=2}^Nh_{k,i}\right)v^\top\right]+\frac{\lambda^2(N-1)}{N^2}vv^\top.
    \end{align*}
    Since $VV^\top$ is a nonzero positive semidefinite matrix, we can let $\lambda\to\infty$ and get $\|\Sigma_W(H^\prime)\|_F\to\infty$.
\end{proof}

\subsection{Proof of Theorem~\ref{thm: mse}}
\label{apx: mse}

\begin{proof}
Without loss of generality, we can assume that $\mu_G=0$, since we can replace $b$ with $b-W\mu_G$.
The loss contributed by the $i$-th datapoint in the $k$-th class can be calculated as
\begin{align*}
    L_{k,i}(W,b)&\triangleq \frac{1}{2}\|Wh_{k,i}+b-e_k\|^2\\
    &=\frac{1}{2}\left[h_{k,i}^\top W^\top W h_{k,i}+2 (b -e_k)^\top W h_{k,i}+(b-e_k)^\top (b-e_k)\right]\\&=\frac{1}{2}\Tr\left[h_{k,i}h_{k,i}^\top W^\top W
\right]+b^\top W h_{k,i}- e_k^\top W 
h_{k,i}+\frac{1}{2}b^\top b-e_k^\top b+\frac{1}{2}.
\end{align*}
The total loss function can be calculated as 
\begin{align*}
    L(W,b)&=\frac{1}{KN}\sum_{k\in[K],i\in[N]} L_{k,i}(W,b)\\
    &=\frac{1}{2}\Tr\left[\Sigma_{T}W^\top W\right]-\frac{1}{K}\sum_{k=1}^K e_k^\top W\mu_k+\frac{1}{2}b^\top b-\frac{1}{K}\mathbf{1}^\top b+\frac{1}{2}.
\end{align*}

The loss function is convex and quadratic, whose optima can be obtained by first order stationary condition. 
The first order condition with regard to $W$ can be expressed as 
\begin{align*}
    \nabla_W L(W,b)=W\Sigma_{T}-\frac{1}{K}\sum_{k=1}^K e_k\mu_k^\top=0.
\end{align*}

To solve this equality, we make a little digress and prove the following bias-variance decomposition:
\begin{align}\label{eq: decomposition}
    \begin{aligned}
    \Sigma_{T}&=\frac{1}{KN}\sum_{k\in[K],i\in[N]} h_{k,i}h_{k,i}^\top\\
    &=\frac{1}{KN}\sum_{k\in[K],i\in[N]}(h_{k,i}-\mu_k+\mu_k)(h_{k,i}-\mu_k+\mu_k)^\top\\
    &=\frac{1}{KN}\sum_{k\in[K],i\in[N]} (h_{k,i}-\mu_k)(h_{k,i}-\mu_k)^\top+
    \frac{2}{KN}\sum_{k\in[K],i\in[N]} (h_{k,i}-\mu_k)\mu_k^\top+\frac{1}{K}\sum_{k\in[K]} \mu_k\mu_k^\top\\
    &=\Sigma_{B}+\Sigma_W,
\end{aligned}
\end{align}
From this we know that $\Sigma_{T}-\Sigma_{B}=\Sigma_W$ is positive semidefinite. This implies that $\mu_1,\cdots,\mu_k $ lies in the column space $V_T$ of $\Sigma_{T}$.

Let $r=\text{rank}(\Sigma_{T})$. There exists a eigenvalue decomposition $\Sigma_{T}=U\Sigma U^\top$, such that $\Sigma=\text{diag}(s_1,\cdots s_r,0,\cdots, 0)$, and $U=(u_1, \cdots u_d)$ satisfying
\begin{align*}
     1.\; \ u_i\perp u_j,\ i\neq j;\quad 2.\; \|u_i\|_2=1;\quad 3. \;u_i\perp \mu_k,\ k\le K,i>K.
\end{align*}
Therefore
\begin{align*}
    \frac{1}{K}\left(\sum_{k=1}^K e_k \mu_k^\top \right) \Sigma_{T}^\dagger \Sigma_{T}&=\frac{1}{K}\left(\sum_{k=1}^K e_k \mu_k^\top \right)U\Sigma^\dagger
    U^\top U\Sigma U^\top\\&=\frac{1}{K}\left(\sum_{k=1}^K e_k \mu_k^\top \right)\cdot \left(\sum_{i=1}^{r}u_i u_i^\top \right)\\&=\frac{1}{K}\left(\sum_{k=1}^K e_k \mu_k^\top \right)-\frac{1}{K}\left(\sum_{k=1}^K e_k \mu_k^\top \right)\cdot \left(\sum_{i=r+1}^d u_i u_i^\top \right)\\&=\frac{1}{K}\left(\sum_{k=1}^K e_k \mu_k^\top \right).
\end{align*}
This implies that $W=\frac{1}{K}\left(\sum_{k=1}^K e_k \mu_k^\top \right)\Sigma_{T}^\dagger$ satisfies the first order optimality condition for $W$. 
It is also easy to see that $b=\frac{1}{K}\mathbf{1}$ satisfies the first order optimality condition of $b$. Therefore, $L(W,b)$ attains its minimum at
\begin{align*}
    W=\frac{1}{K}\left(\sum_{k=1}^K e_k \mu_k^\top \right)\Sigma_{T}^\dagger,\quad b=\frac{1}{K}\mathbf{1}, 
\end{align*}
with optimal value 
\begin{align*}
    \min_{W,b}L(W,b)&=\frac{1}{2}\Tr\left[\Sigma_{T}\cdot \Sigma_{T}^\dagger\cdot\frac{1}{K}\left(\sum_{k=1}^K \mu_k e_k^\top \right)\cdot \frac{1}{K}\left(\sum_{k=1}^K e_k \mu_k^\top \right)\Sigma_{T}^\dagger\right]
    \\&\quad-\Tr\left[\frac{1}{K}\left(\sum_{k=1}^K \mu_k e_k^\top \right)\cdot \frac{1}{K}\left(\sum_{k=1}^K e_k \mu_k^\top \right)\Sigma_{T}^\dagger\right]+\frac{1}{2}-\frac{1}{2K}\\
    &=-\frac{1}{2K}\Tr
    \left[\frac{1}{K}\left(\sum_{k=1}^K \mu_k\mu_k^\top \right)\cdot \Sigma_{T}^\dagger\right]+\frac{1}{2}-\frac{1}{2K}\\&=-\frac{1}{2K} \Tr\left[\Sigma_{T}^\dagger \Sigma_{B}\right]+\frac{1}{2}-\frac{1}{2K}
\end{align*}
where we use $\Sigma_{T}^\dagger \Sigma_{T} \Sigma_{T}^\dagger=\Sigma_{T}^\dagger$ in the second equality.
\end{proof}

\subsection{Proof for Theorem~\ref{thm: upper bound}}
\label{apx: upper bound}
We need the following lemmas on block matrices.
\begin{lemma}\label{lem: inv}
Let $A\in\BR^{d_1\times d_1},B\in\BR^{d_1\times d_2},C\in\BR^{d_2\times d_1},D\in\BR^{d_2\times d_2}$. If $\left[\begin{array}{ll}
{A} & {B} \\
{C} & {D}
\end{array}\right]$ and $D$ are invertible, then ${A}-{B D}^{-1} {C}$ is invertible and 
\begin{align}\label{eq: inv}
    \left[\begin{array}{ll}
{A} & {B} \\
{C} & {D}
\end{array}\right]^{-1}=\left[\begin{array}{cc}
\left({A}-{B D}^{-1} {C}\right)^{-1} & -\left({A}-{B D}^{-1} {C}\right)^{-1} {B D}^{-1} \\
-{D}^{-1} {C}\left({A}-{B D}^{-1} {C}\right)^{-1} & {D}^{-1}+{D}^{-1} {C}\left({A}-{B D}^{-1} {C}\right)^{-1} {B D}^{-1}
\end{array}\right].
\end{align}
\end{lemma}

\begin{proof}
The invertibility of $A$ and $D$ are obvious. The following identity gives the invertibility of ${A}-{B D}^{-1} {C}$:
\begin{align*}
    \left[\begin{array}{cc}
A & B \\
C & D
\end{array}\right]\left[\begin{array}{cc}
\mathbf{I} & \mathbf{0} \\
-D^{-1} C & \mathbf{I}
\end{array}\right]=\left[\begin{array}{cc}
A-B D^{-1} C & B \\
\mathbf{0} & D
\end{array}\right]
\end{align*}
The equation~\ref{eq: inv} can be check by direct calculation.

\end{proof}

\begin{lemma}\label{lem: positive}
    Let $A\in\BR^{d_1\times d_1},B\in\BR^{d_1\times d_2},C\in\BR^{d_2\times d_2}$. If $\left[\begin{array}{ll}
{A} & {B} \\
{B}^\top  & {C}
\end{array}\right]\succ \mathbf{0}$, then  ${A}-{B C}^{-1} {B}^\top\succ\mathbf{0}$.
\end{lemma}
\begin{proof}
It is the direct consequence of the following identity.
\begin{align*}
    \begin{aligned}
&\left(\begin{array}{cc}
\mathbf{I} & \mathbf{0} \\
-B^\top A^{-1} & \mathbf{I}
\end{array}\right)\left(\begin{array}{cc}
A & B \\
B^\top & C
\end{array}\right)\left(\begin{array}{cc}
\mathbf{I} & \mathbf{0} \\
-B^\top A^{-1} & \mathbf{I}
\end{array}\right)^\top= 
\left(\begin{array}{cc}
A & \mathbf{0} \\
\mathbf{0} & C-B^\top A^{-1} B
\end{array}\right)
\end{aligned}
\end{align*}
\end{proof}

\begin{proof}[Proof of Theorem~\ref{thm: upper bound}]

Let $r=\rank(\Sigma_{T})$. There exists an eigenvalue decomposition $\Sigma_{T}=U\left[\begin{array}{ll}
\Sigma&  \\
 & \mathbf{0}
\end{array}\right] U^\top$ with $\Sigma=\text{diag}(s_1,\cdots s_r)$. From Equation~\ref{eq: decomposition}, we know that $V_B$, the column space of $\Sigma_{B}$ is a subspace of $V_T$, the column space of $\Sigma_{T}$. Therefore, there exists $W\in \BR^{r\times r}$, such that $\Sigma_{B}=U\left[\begin{array}{ll}
W& \\
 & \mathbf{0}
\end{array}\right]U^\top$. This implies that 
\begin{align*}
    \Tr\left[\Sigma_{T}^\dagger \Sigma_{B}\right]
    &=\Tr\left[(U\left[\begin{array}{ll}
\Sigma&  \\
 & \mathbf{0}
\end{array}\right] U^\top)^\dagger U\left[\begin{array}{ll}
W&  \\
 & \mathbf{0}
\end{array}\right]U^\top \right]\\
&=\Tr\left[U\left[\begin{array}{ll}
\Sigma& \\
 & \mathbf{0}
\end{array}\right]^\dagger U^\top U\left[\begin{array}{ll}
W&  \\
 & \mathbf{0}
\end{array}\right]U^\top \right]\\
&=\Tr[\Sigma^{-1} W]
\end{align*}

Let $r_1=\rank(W)$. Denote $W=U_1 \left[\begin{array}{ll}
\Sigma_1& \\
 & \mathbf{0}
\end{array}\right] U_1^\top$ as the eigenvalue decomposition of $W$. Denote
$V=\left[\begin{array}{ll}
V_1& V_2 \\
V_2^\top & V_3
\end{array}\right]=U_1^\top \Sigma U_1$, where $V_1\in\BR^{r_1\times r_1}$.
Since $V\succ \mathbf{0}$, we can evoke Lemma~\ref{lem: inv} have
\begin{align*}
    \Tr\left[\Sigma^{-1}W\right]
    &=\Tr\left[U_1 V^{-1} U_1^\top U_1 \left[\begin{array}{ll}
\Sigma_1& \\
 & \mathbf{0}
\end{array}\right] U_1^\top\right]\\
&=\Tr\left[\left(V_1-V_2^\top V_3^{-1}V_2\right)^{-1}\Sigma_1\right]
\end{align*}

From Equation~\ref{eq: decomposition}, we know that $W\preceq \Sigma$. Use Lemma~\ref{lem: positive}, we get $\mathbf{0}\prec\Sigma_1\preceq V_1-V_2^\top V_3^{-1}V_2$. This implies that 
\begin{align*}
    \Tr\left[\left(V_1-V_2^\top V_3^{-1}V_2\right)^{-1}\Sigma_1\right]
    &=r_1-\Tr\left[\left(V_1-V_2^\top V_3^{-1}V_2\right)^{-1}\left(V_1-V_2^\top V_3^{-1}V_2-\Sigma_1\right)\right]\le r_1,
\end{align*}
where in the last inequality, we use the fact that the trace of the product of two symmetric positive semidefinite matrices is nonnegative. Therefore, we obtain the inequality that 
$$\Tr[\Sigma_{T}^\dagger \Sigma_{B}]\le \rank(\Sigma_{B}).$$

For fully collapsed configuration, we have $\Sigma_B=\Sigma_T$, and the equality is attained.


\end{proof}

\section{Additional Experimental Results in \cref{sec: col_exp}}\label{sec: addional exp}

We show in \cref{fig: traintest-prev} the test collapse for Fuzziness, Squared Distance and Cosine Similarity.
\begin{figure*}[t]
     \centering
     \begin{subfigure}
         \centering
         \includegraphics[width=0.3\textwidth]{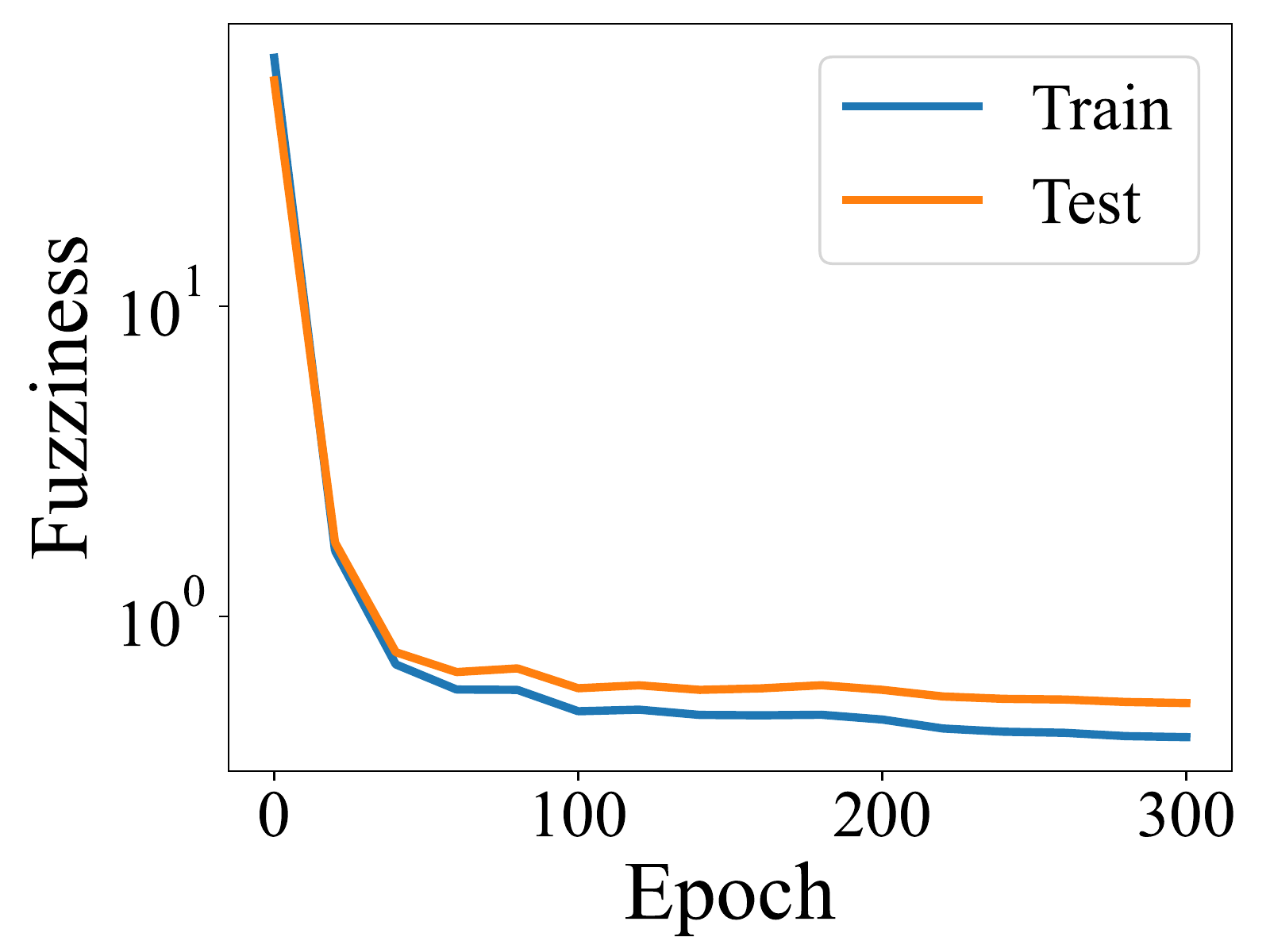}
     \end{subfigure}
     \hfill
     \begin{subfigure}
         \centering
         \includegraphics[width=0.3\textwidth]{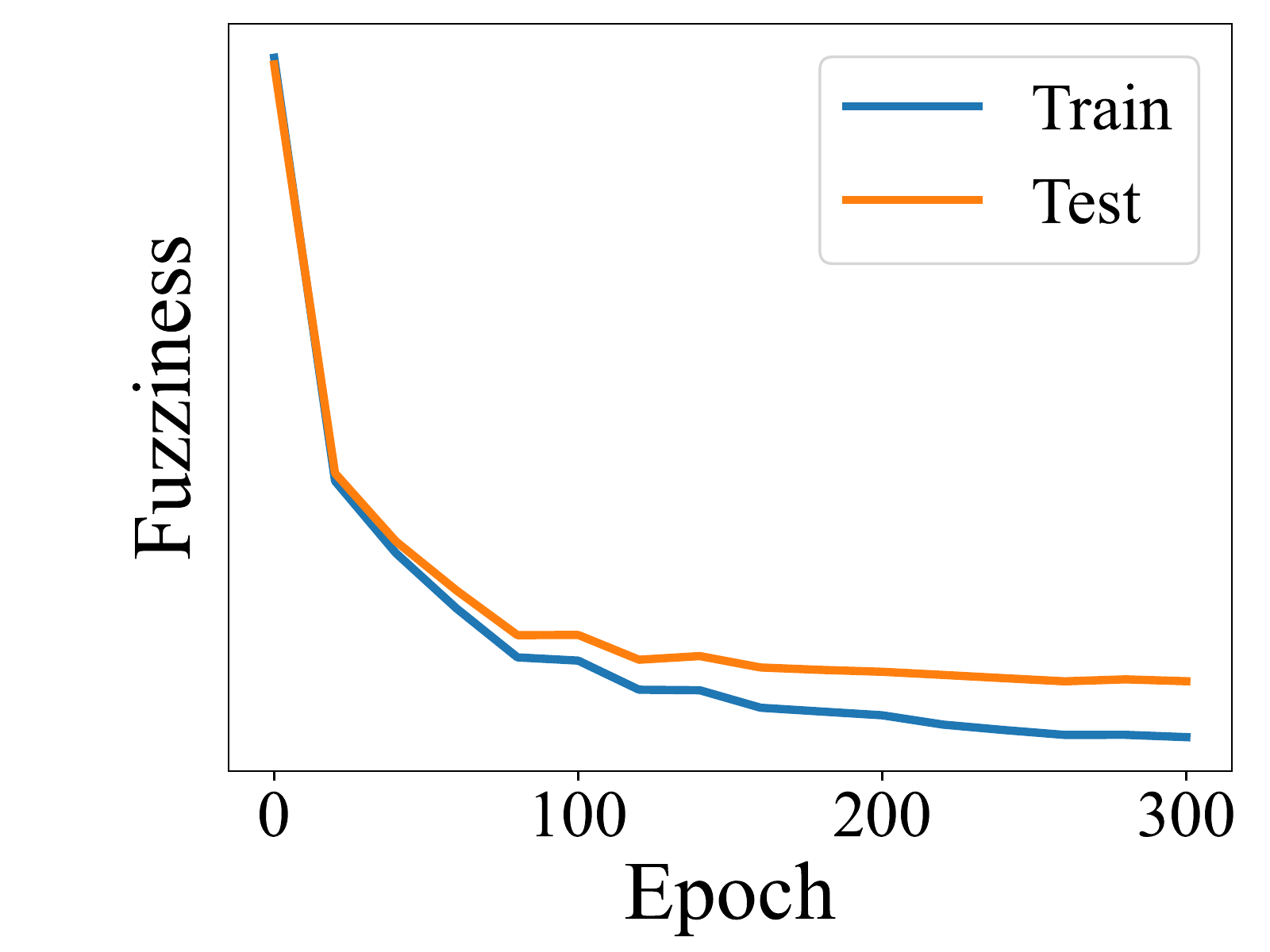}
     \end{subfigure}
     \hfill
     \begin{subfigure}
         \centering
         \includegraphics[width=0.3\textwidth]{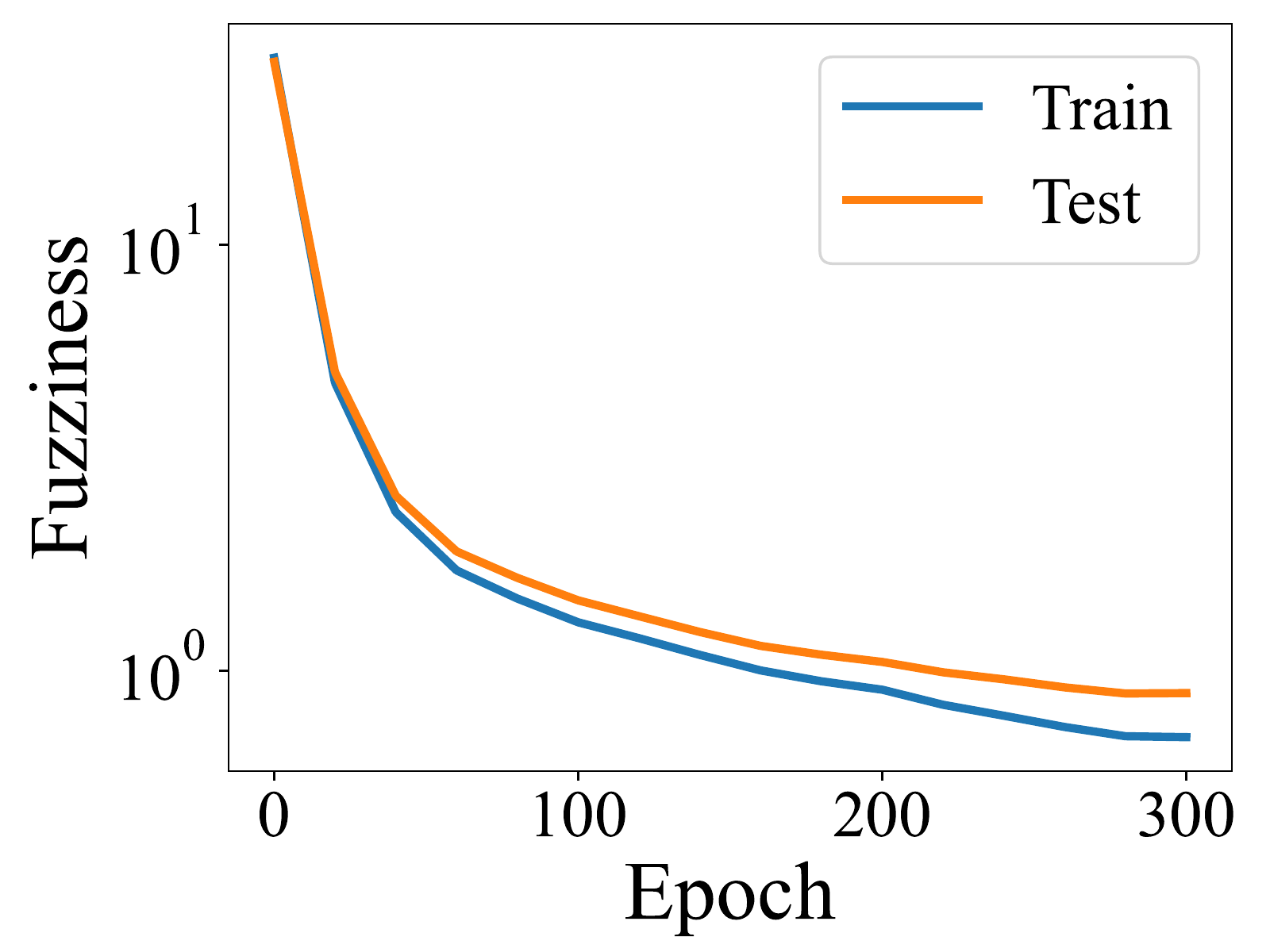}
     \end{subfigure} \\
     \centering
     \begin{subfigure}
         \centering
         \includegraphics[width=0.3\textwidth]{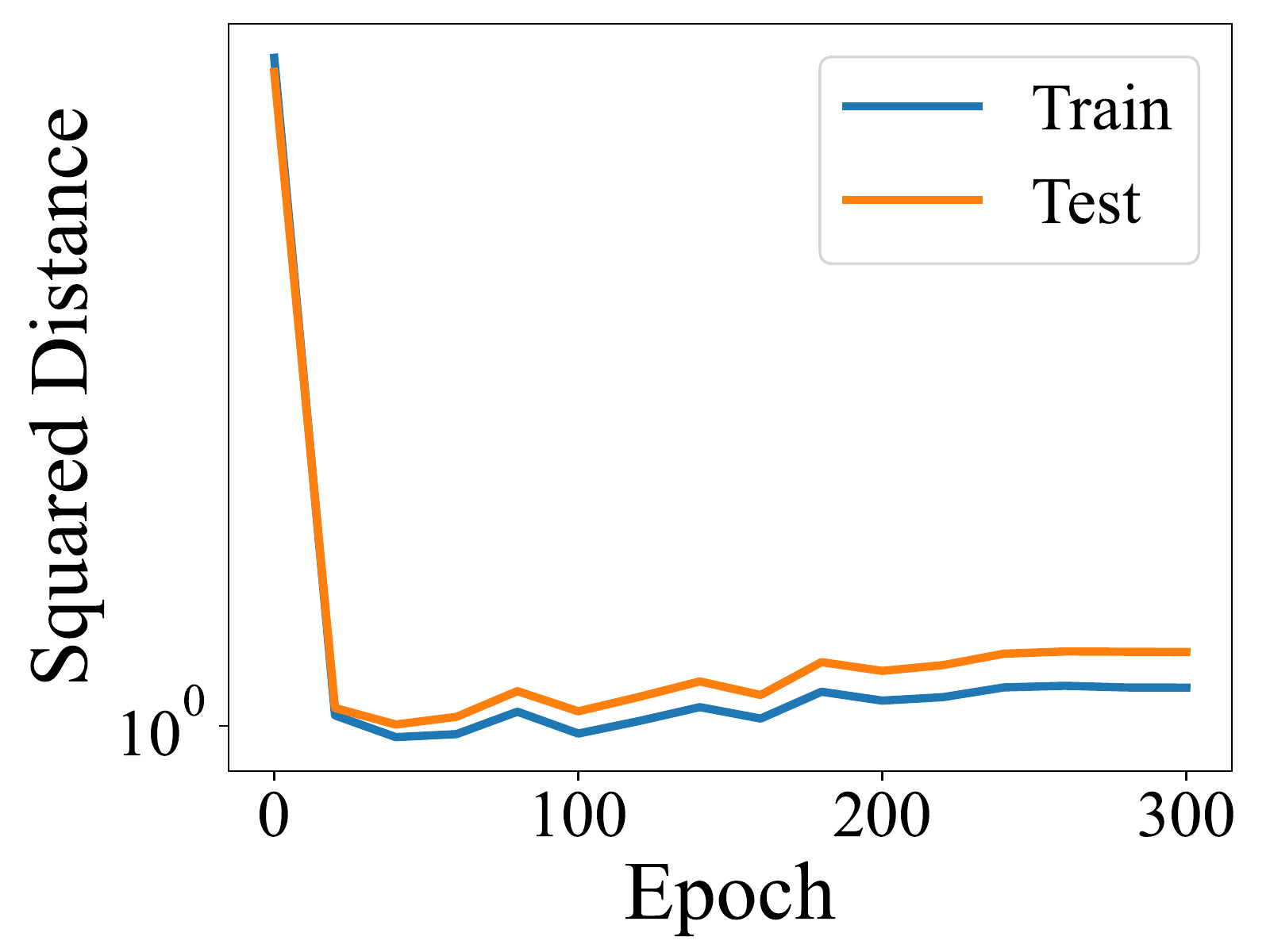}
     \end{subfigure}
     \hfill
     \begin{subfigure}
         \centering
         \includegraphics[width=0.3\textwidth]{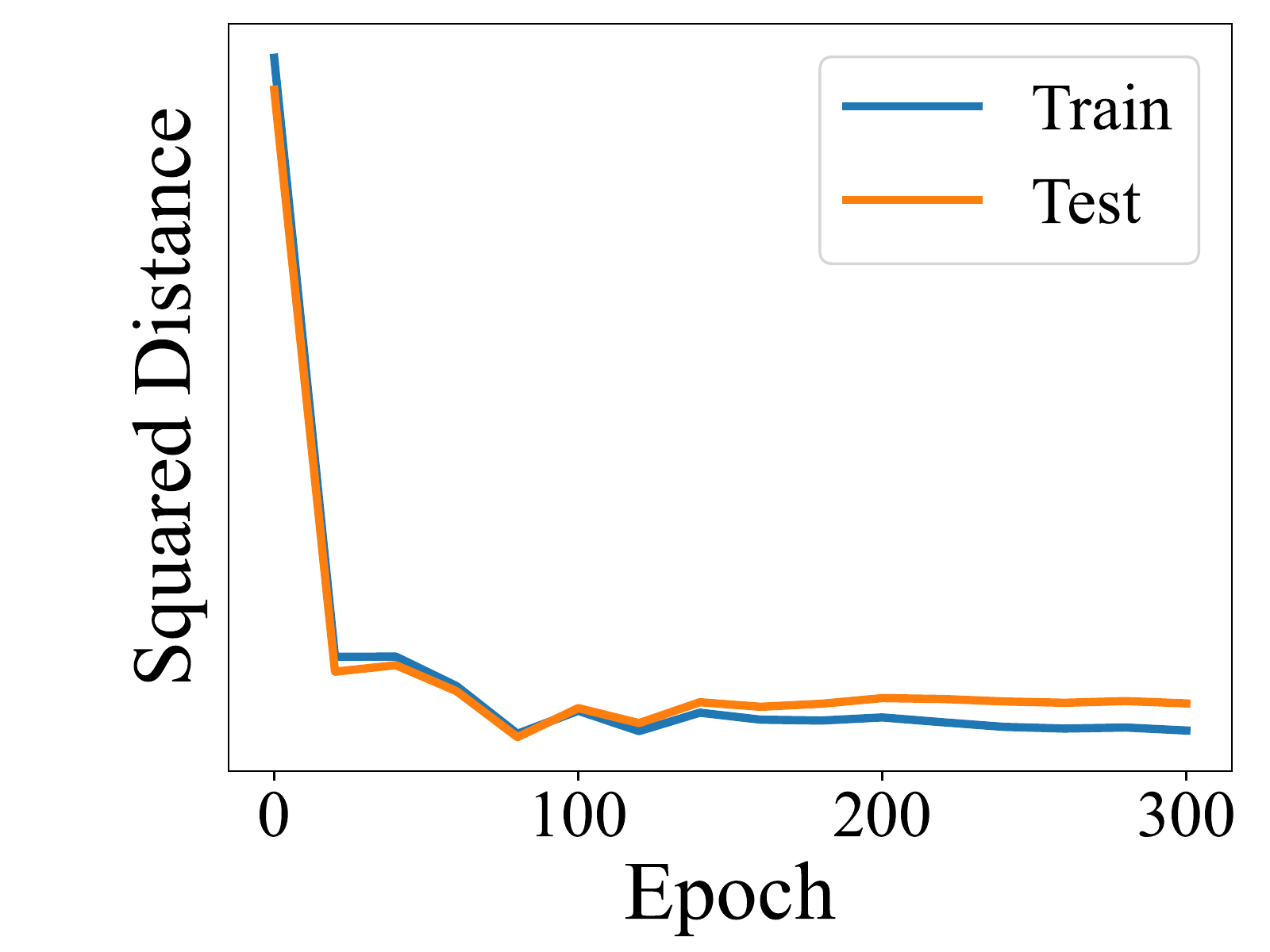}
     \end{subfigure}
     \hfill
     \begin{subfigure}
         \centering
         \includegraphics[width=0.3\textwidth]{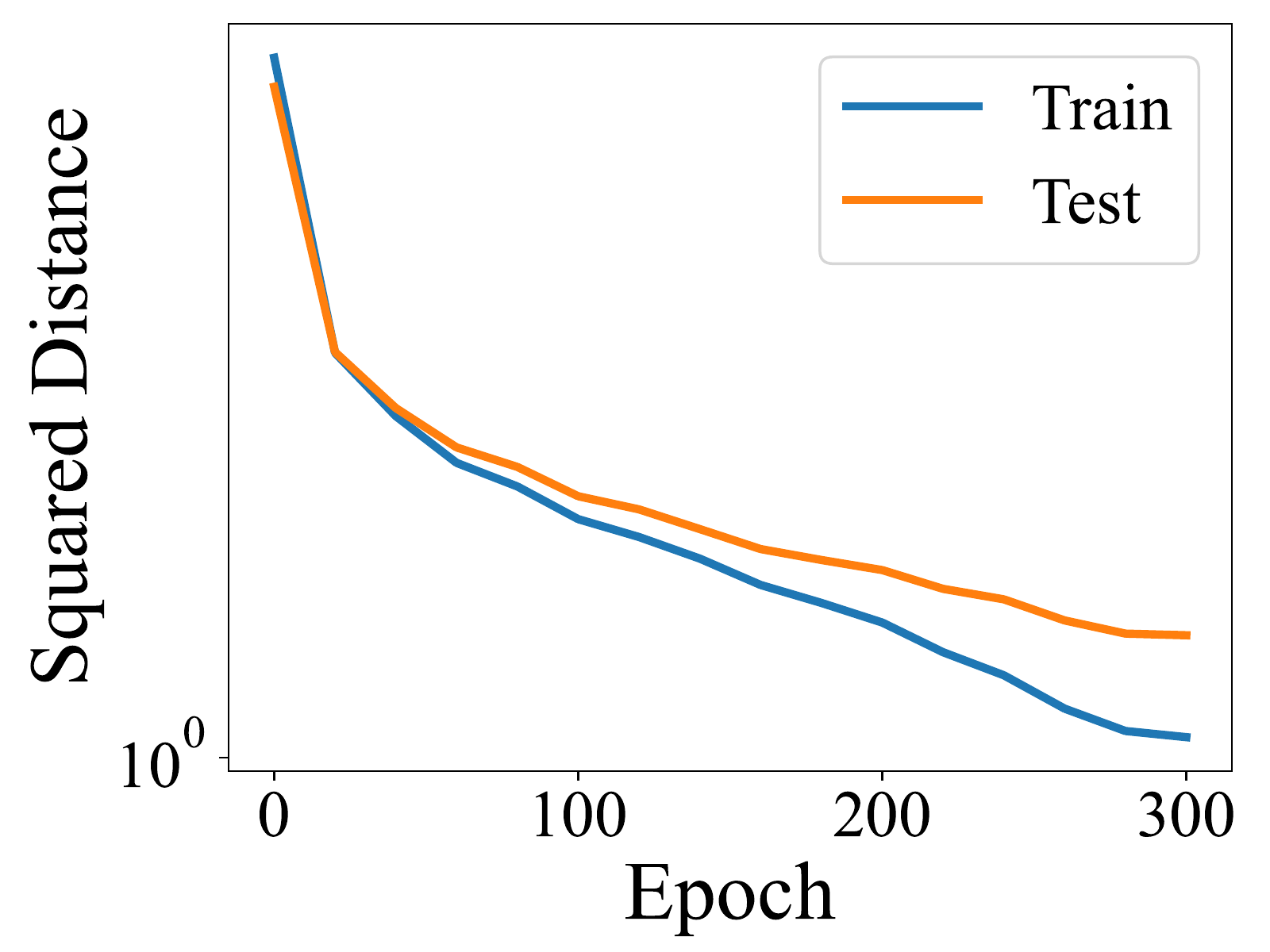}
     \end{subfigure} \\
     \centering
     \begin{subfigure}
         \centering
         \includegraphics[width=0.3\textwidth]{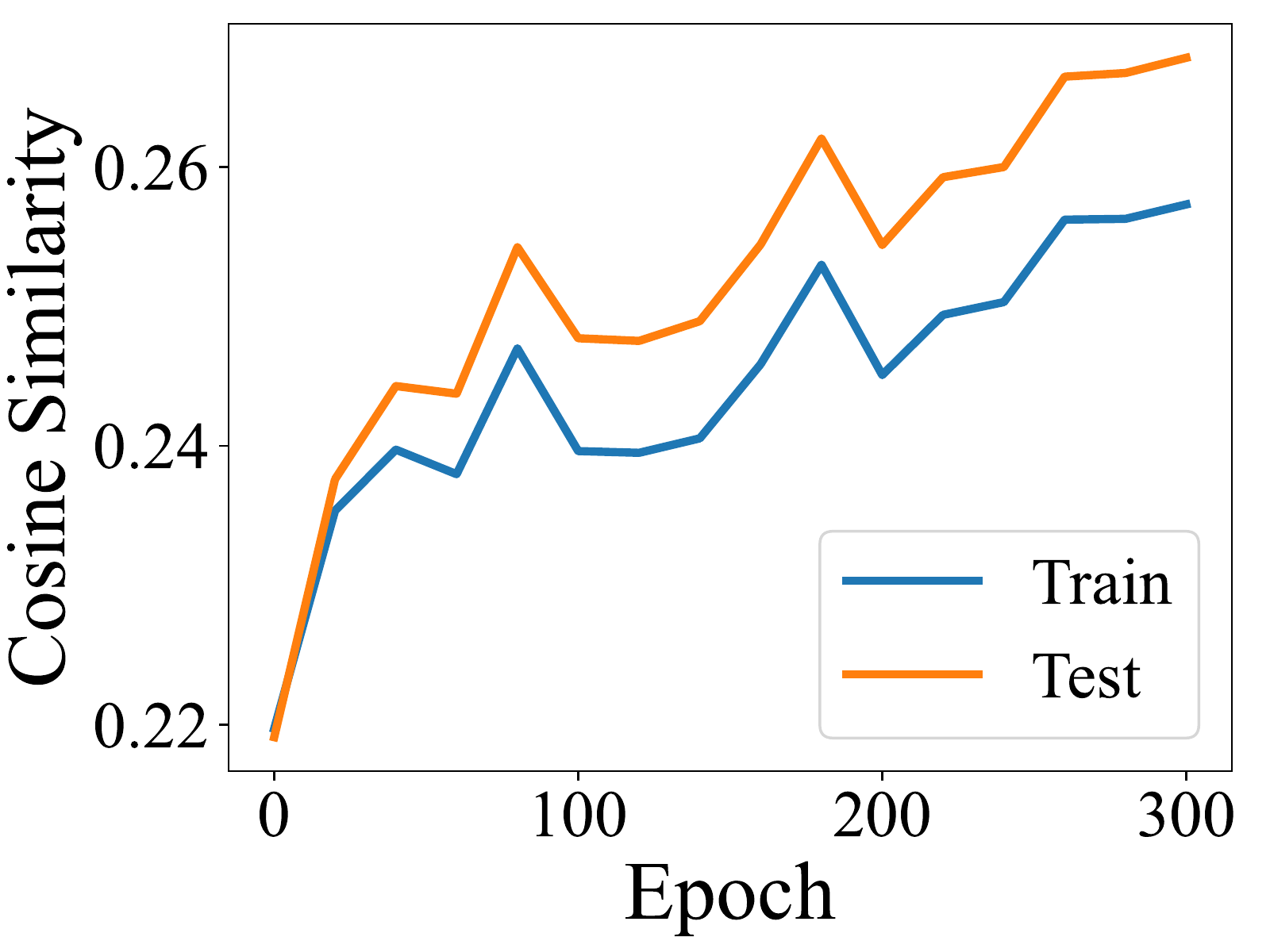}
     \end{subfigure}
     \hfill
     \begin{subfigure}
         \centering
         \includegraphics[width=0.3\textwidth]{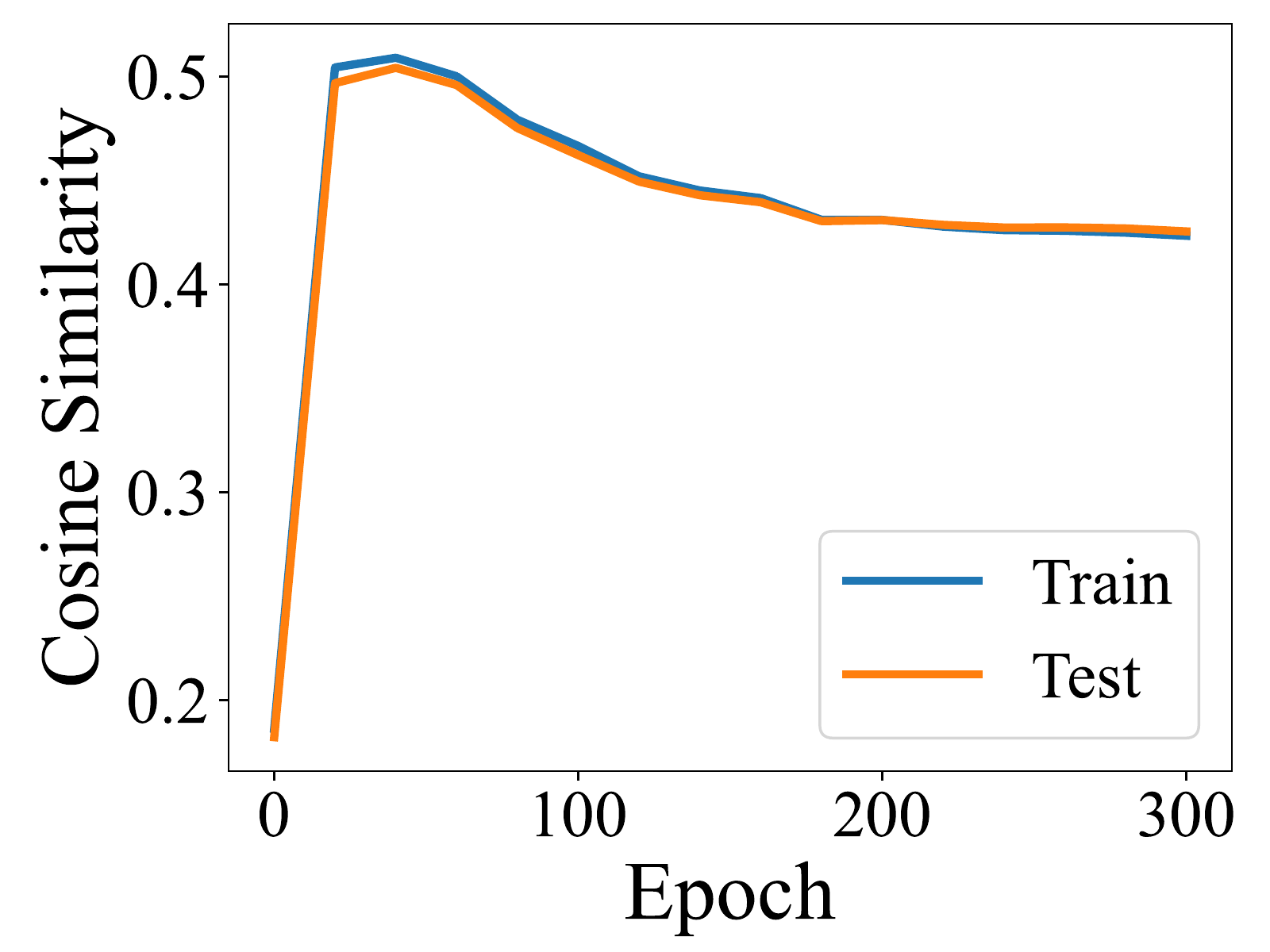}
     \end{subfigure}
     \hfill
     \begin{subfigure}
         \centering
         \includegraphics[width=0.3\textwidth]{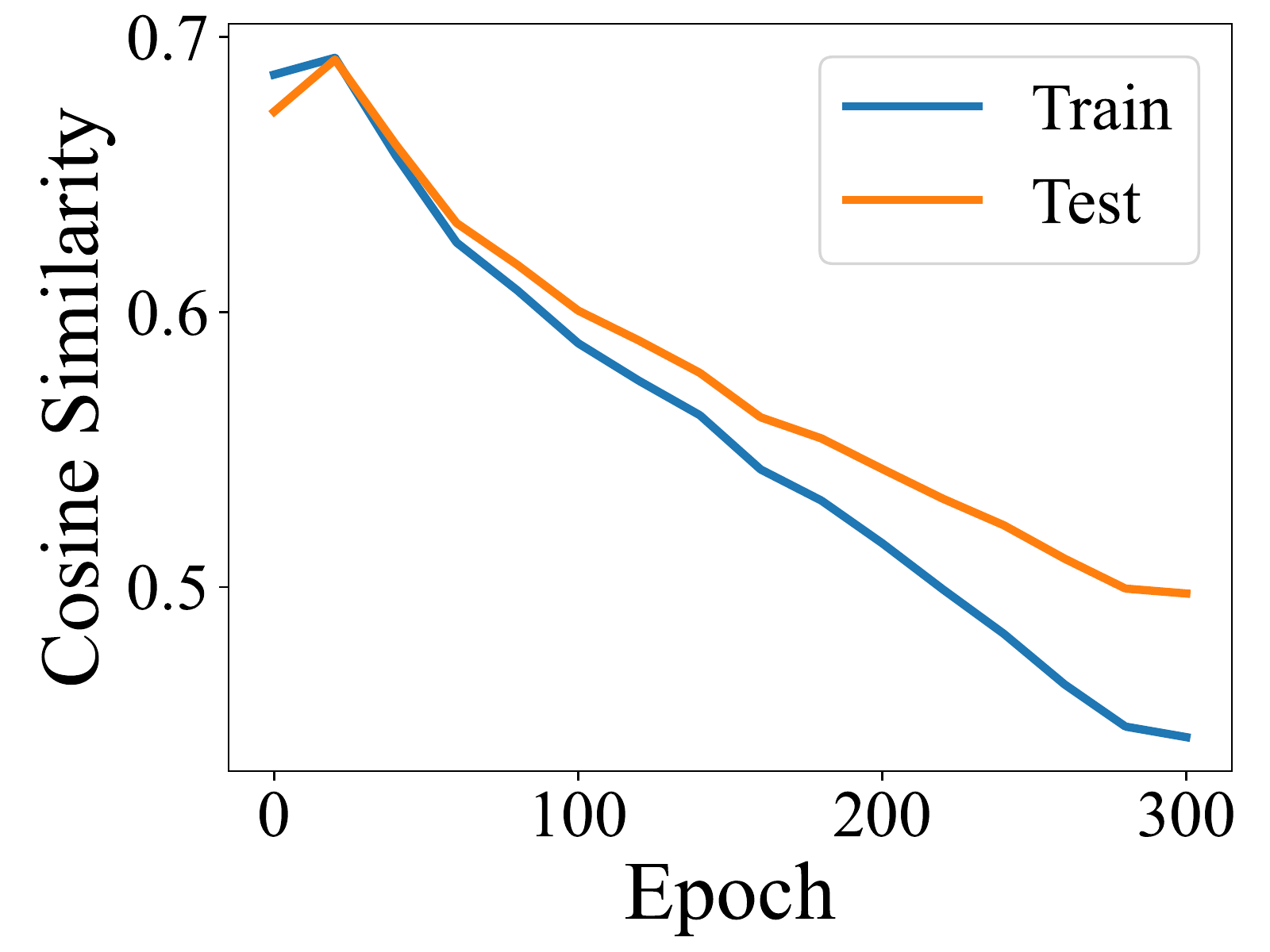}
     \end{subfigure} \\
        \caption{\textbf{Additional Experiments on the Comparisons of Train Collapse and Test Collapse for Previous Variability Collapse Metrics.} Train collapse is evaluated on a 50000 subset of ImageNet-1K training dataset. \textbf{Top Row:} Fuzziness. \textbf{Middle Row:} Squared Distance. \textbf{Bottom Row:} Cosine Similarity. \textbf{Left:} ResNet18 on CIFAR-10. \textbf{Middle:} ResNet50 on ImageNet-1K. \textbf{Orange:} DeiT-S on ImageNet-1K.} 
        \label{fig: traintest-prev}
\end{figure*}

\section{Raw Experiment Data in \cref{sec: trans_exp}}\label{sec: raw_data}
See \cref{tab: trans_raw_data} and \cref{tab: trans_graph_data} for raw data in downstread classification experiments and pretraining experiments, respectively. 
For Aircraft, Pets, Caltech101 and Flowers, we use mean per-class accuracy. \cite{chen2020simple} For other datasets, we use top-1 accuracy.

\begin{table}[t]
\caption{Raw data of downstream classification experiments in \cref{sec: trans_exp}. The meanings of $\tau$ and $\lambda$ are introduced in the main text. MLO = mean log odds.}
\vskip 0.15in
\begin{center}
\begin{small}
\begin{sc}
\begin{tabular}{lp{1.0cm}p{1.0cm}p{1.0cm}p{1.0cm}p{1.0cm}p{1.0cm}p{1.0cm}p{1.0cm}p{1.0cm}p{1.0cm}p{1.0cm}p{1.0cm}p{1.0cm}p{1.0cm}p{1.0cm}p{1.0cm}p{1.0cm}}
\toprule
        Param & Pets & Flowers & Aircraft & Cars & DTD & Cifar10 & Food101 & Cifar100 & Caltech & SUN397 & MLO \\ \midrule
        $\tau=0.1$ & 90.8 & 95.03 & 58.85 & 64.73 & 74.47 & 93.1 & 75.45 & 77.14 & 84.9 & 62.44 & 1.445 \\ 
        $\tau=0.2$ & 91.52 & 95.25 & 57.5 & 64.1 & 74.41 & 92.89 & 75.41 & 77.65 & 85.58 & 63.28 & 1.459 \\ 
        $\tau=0.4$ & 91.64 & 94.66 & 57.67 & 64.42 & 74.63 & 92.58 & 74.97 & 77.1 & 85.59 & 63.3 & 1.441 \\ 
        $\tau=0.6$ & 91.74 & 94.14 & 57.54 & 61.56 & 76.22 & 92.42 & 74.86 & 76.6 & 85.2 & 63.51 & 1.421 \\ 
        $\tau=0.8$ & 92.35 & 93.52 & 55.64 & 61.6 & 75.37 & 92.37 & 74.2 & 76.07 & 84.49 & 62.74 & 1.39 \\ 
        $\tau=1.0$ & 91.7 & 93.68 & 55.37 & 60.14 & 73.99 & 91.9 & 74.42 & 75.96 & 85.82 & 62.98 & 1.375 \\ 
        $\lambda=2.5$ & 90.86 & 92.01 & 51.07 & 55.14 & 74.36 & 91.79 & 72.9 & 75.44 & 85.19 & 62.01 & 1.282 \\ 
        $\lambda=3.75$ & 90.87 & 90.79 & 47.81 & 50.44 & 76.01 & 91.89 & 72.19 & 75.31 & 82.97 & 60.98 & 1.22 \\ 
        $\lambda=5$ & 90.51 & 90.03 & 45.12 & 50.52 & 74.41 & 91.63 & 71.84 & 74.64 & 82.63 & 60.38 & 1.174 \\ 
        $\lambda=7.5$ & 89.33 & 86.54 & 44.52 & 46.87 & 73.14 & 91.64 & 71.19 & 74.13 & 81.23 & 58.55 & 1.08 \\ 
        $\lambda=10$ & 89.92 & 83.6 & 43.39 & 43.55 & 72.5 & 91.43 & 70.29 & 73.76 & 78.25 & 56.9 & 1.008 \\ 
\bottomrule
\end{tabular}
\end{sc}
\end{small}
\end{center}
\vskip -0.1in
    \label{tab: trans_raw_data}
\end{table}

\begin{table}[t]
\caption{Raw data of pretraining runs in \cref{sec: trans_exp}.}
\vskip 0.15in
\begin{center}
\begin{small}
\begin{sc}
\begin{tabular}{lcccccc}
\toprule
        Param & Fuzziness & Sqr Dist & Cos Sim & VCI & Accuracy & MLOG \\ \midrule
        $\tau=0.1$ & 15.93 & 2.829 & 0.4634 & 0.796 & 75.4 & 0.3250 \\ 
        $\tau=0.2$ & 15.64 & 2.534 & 0.4856 & 0.7849 & 76.64 & 0.2704 \\ 
        $\tau=0.4$ & 13.45 & 2.141 & 0.512 & 0.7552 & 77.53 & 0.2029 \\ 
        $\tau=0.6$ & 12.62 & 1.923 & 0.5371 & 0.7406 & 77.72 & 0.1711 \\ 
        $\tau=0.8$ & 11.83 & 1.742 & 0.5589 & 0.728 & 77.83 & 0.1343 \\ 
        $\tau=1.0$ & 11.46 & 1.625 & 0.5767 & 0.7201 & 77.88 & 0.1161 \\ 
        $\lambda=2.5$ & 13.15 & 2.78 & 0.4097 & 0.7216 & 77.52 & 0.0440 \\ 
        $\lambda=3.75$ & 14.31 & 3.803 & 0.3556 & 0.719 & 77.31 & -0.0063 \\ 
        $\lambda=5$ & 14.51 & 4.783 & 0.3551 & 0.7099 & 77.29 & -0.0507 \\ 
        $\lambda=7.5$ & 15.22 & 7.287 & 0.2794 & 0.6936 & 77.46 & -0.1539 \\ 
        $\lambda=10$ & 13.92 & 9.964 & 0.172 & 0.6652 & 77.65 & -0.2375 \\ 
\bottomrule
\end{tabular}
\end{sc}
\end{small}
\end{center}
\vskip -0.1in
    \label{tab: trans_graph_data}
\end{table}

\end{document}